\documentclass[12pt]{article}
\usepackage[a4paper,
            bindingoffset=0.2in,
            left=1in,
            right=1in,
            top=1in,
            bottom=1in,
            footskip=.25in]{geometry}

\usepackage{url}
\usepackage{multirow}
\usepackage{amsmath}
\usepackage{amssymb}
\usepackage{csquotes}
\usepackage{graphicx}
\usepackage{booktabs}
\usepackage{hyperref}

\usepackage{amsthm}
\usepackage{thmtools} 
\usepackage{thm-restate}

\usepackage{color}
\newtheorem{definition}{Definition}
\newtheorem{theorem}{Theorem}

\newtheorem{corollary}{Corollary}

\newcommand{\ev}[1] {
    \text{E}[#1]
}

\begin{document}


\title{Decomposing Hard SAT Instances with Metaheuristic Optimization}

\author{Daniil Chivilikhin \and Artem Pavlenko \and  Alexander Semenov*
\\
ITMO University, St. Petersburg, Russia\\
\href{mailto: alex.a.semenov@itmo.ru}{\texttt{ alex.a.semenov@itmo.ru}}}

\maketitle


\begin{abstract}
\noindent \emph{In the article, within the framework of the Boolean Satisfiability problem (SAT), the problem of estimating the hardness of specific Boolean formulas w.r.t. a specific complete SAT solving algorithm is considered. 
Based on the well-known Strong Backdoor Set (SBS) concept, we introduce the notion of decomposition hardness (d-hardness). 
If $B$ is an arbitrary subset of the set of variables occurring in a SAT formula $C$, and $A$ is an arbitrary complete SAT solver , then the d-hardness expresses an estimate of the hardness of $C$ w.r.t. $A$ and $B$. 
We show that the d-hardness of $C$ w.r.t. a particular $B$ can be expressed in terms of the expected value of a special random variable associated with $A$, $B$, and $C$. 
For its computational evaluation, algorithms based on the Monte Carlo method can be used. 
The problem of finding $B$ with the minimum value of d-hardness is formulated as an optimization problem for a pseudo-Boolean function whose values are calculated as a result of a probabilistic experiment. 
To minimize this function, we use evolutionary algorithms. 
In the experimental part, we demonstrate the applicability of the concept of d-hardness and the methods of its estimation to solving hard unsatisfiable SAT instances.
}
\end{abstract}


\section{Introduction}

The Boolean satisfiability problem~(SAT) is an NP-complete combinatorial problem~\cite{Cook71} in its decision variant.
The NP-completeness of SAT, proved in Steven Cook's seminal work~\cite{Cook71}, gave birth to the modern theory of structural complexity of algorithms. 
A number of classical algorithms for solving SAT were proposed in the 1960s and have been used for a long time in automated theorem proving~\cite{DP1960,DPLL1962,Robins1965}. 
There are also a number of SAT solving algorithms that exploit local search methods and approaches, see e.g.~\cite{Selman1992,Gu1996,Schoning1999,Balint2012,Cai2021}. 
However, the widespread practical use of SAT solvers began with the advent of the CDCL (Conflict Driven Clause Learning) algorithm~\cite{MSS96,MSS1999}. 
Over the past 20 years, CDCL-based complete SAT solving algorithms have proven their efficiency in many practically important industrial applications. 
Currently, SAT solvers are successfully applied to combinatorial problems from hardware and software verification~\cite{BiereDAC99,BiereTACAS99,Kroening09,ABC-Misch10}, verification of neural networks~\cite{narod18-aaai,narod18-ijcai,jia20}, bioinformatics~\cite{LynceMS06}, cryptanalysis~\cite{CourtBard07,Bard09,SZBP11,SZOKI18}, combinatorics and Ramsey theory~\cite{KonLisitsa14,Heule16,Heule18}, and in many other areas.

Unfortunately, quite often in practice a SAT solver may work on some input formula for a long time, and there is no information about how much more time it needs to complete the work. 
Sometimes, minimal additional information (for example, substituting the value of some variable or variables) can radically change the situation, and the solver will complete the work in a few seconds. 
These effects are the consequences of the well-known heavy-tailed behavior phenomenon~\cite{GomesSabh2021}. 
And with this in mind, the problem of constructing meaningful estimates of how hard a particular formula is for a particular SAT solving algorithm $A$ becomes important.

There are quite a few papers on the topic of analytical estimates of the hardness of SAT solving algorithms on countable families of formulas, see e.g.~\cite{Alekhn02,benSasson2001,benSasson2004,Buss1998,bonet1998,Haken1985,Tse70,urquhart1995}. 
However, to our best knowledge, the problem of estimating the hardness of a specific formula w.r.t. a specific SAT solver is still far from its solution.

In this article, we propose a general approach to solving this problem.
The basis of this approach is the idea of obtaining hardness estimates not for the original formula, but for some decomposition of this formula. 
It is often easier to assess the complexity of a decomposition, since the formulas that make it up turn out to be much easier for the SAT solver in comparison with the original formula. 
In fact, this idea is inspired by the concept of the Strong Backdoor Set~\cite{Williams03} and the results of the paper~\cite{Ansoteg08}, in which the concept of backdoor hardness was implicitly formulated.
A set $B$ is a strong backdoor set~(SBS) for formula $C$ w.r.t. a polynomial-time algorithm $P$, if all $2^{|B|}$ subformulas derived from $C$ by substituting the values of variables from $B$ are solved by $P$.
And if $B$ is some SBS for formula $C$ w.r.t. a polynomial sub-solver $P$, then the following upper bound for $H(C)$ (hardness of $C$) holds: $H(C) \le poly(|C|) \cdot 2^{|B|}$, where $poly(\cdot)$ is some polynomial. 
Thus, we are interested in the SBS of the smallest cardinality (minimum SBS), which gives the lowest hardness estimate. 
The problem of finding the minimum SBS, however, is very difficult.

Based on the ideas of~\cite{Ansoteg08}, we introduce the concept of decomposition hardness (d-hardness): this is the running time of a complete SAT solver $A$ on formulas from some family obtained as a result of a decomposition of the original formula $C$. 
We show that d-hardness of formula $C$ w.r.t. an arbitrary complete SAT solver $A$ and a specific decomposition set $B$, $B \subseteq X$, can be expressed via expected value of a special random variable which is associated with $B$ (here, $X$ is the set of variables over which formula $C$ is specified). 
For each $B, B \subseteq X$, the d-hardness of $C$ w.r.t. $B$ gives an upper bound of the hardness of formula $C$ (similar to the case when $B$ is some SBS).
The estimation of d-hardness w.r.t. a concrete set $B$ can be done using the Monte Carlo method. 
Next, we propose to search for a set $B$ which gives a minimal d-hardness value with the help of metaheuristic algorithms used in pseudo-Boolean optimization.

This paper significantly extends the results previously published in conference proceedings \cite{CP2021} with the following unique contributions.
\begin{enumerate}
    \item We study the choice of solver performance metrics for d-hardness estimation. 
    In particular, we consider and compare d-hardness estimation based on execution time, the number of unit propagations, and conflicts.
    
    \item We show how to use the proposed decomposition approach to build an unsatisfiability proof comprised of a set of independent subproofs that can be checked in parallel.
    
    \item We draw a connection between d-hardness~\cite{CP2021} and the $\rho$-backdoor concept proposed in~\cite{AAAI2022}. Based on the idea of the $\rho$-backdoor, we suggest a new technique for d-hardness estimation based on taking into account simple subproblems in a specific decomposition. We show that this technique significantly increases the efficiency of d-hardness estimation. 
     
    \item We show how to use several decomposition sets $B$ simultaneously to achieve a lower d-hardness estimation, smaller solving time, and faster proof checking.
\end{enumerate}

This paper is structured as follows.
In Section~\ref{sec:preliminaries} we introduce basic concepts and theoretical facts used throughout the paper.
The motivation for the research behind this paper is formulated in Section~\ref{sec:motivation}.
Section~\ref{sec:d-hardness} introduces the decomposition hardness of SAT instances. 
Section~\ref{sec:d-hardness-search} 
describes how to formulate objective functions for using metaheuristic algorithms to search for sets $B$ with low d-hardness estimation and which metaheuristic algorithms we use for this purpose.
Section~\ref{sec:search-space-reduction} contains a description of search space reduction techniques which are crucially important for efficient search for decomposition sets. 
In Section~\ref{sec:additional} we describe modifications that make the d-hardness estimation more efficient.
In Section~\ref{sec:proofs} we explain how to use an arbitrary set $B$ to generate and check a decomposed unsatisfiability proof.
Experimental results are reported in Section~\ref{sec:experiments}.
Related work is overviewed in Section~\ref{sec:related-work}, and Section~\ref{sec:conclusion} concludes the paper.

\section{Preliminaries}
\label{sec:preliminaries}
\subsection{Boolean satisfiability problem and SAT solvers}
Recall that Boolean variables have values from the set $\{0, 1\}$.
Let $X$ be an arbitrary set of Boolean variables.
An arbitrary total function $\alpha \colon X \rightarrow \{0, 1\}$ defines an \emph{assignment} of variables from $X$.
The set of all binary words of length $k$ is denoted as $\{0, 1\}^k$.
If we want to emphasize that $\{0, 1\}^k$ is the set of all assignments of variables from $X$, we use the notation $\{0, 1\}^{|X|}$.

A \emph{Boolean formula} is a word constructed according to some rules over the alphabet that includes Boolean variables, parentheses, and special symbols called \emph{logical connectives}: conjunction, disjunction, negation, sum modulo two, etc.
Formulas $x$ and $\neg x$, where $x$ is a Boolean variable, are called \emph{literals} of $x$.
Literals $x$ and $\neg x$ are called \emph{complementary}.
A \emph{clause} is a disjunction of literals, which does not include complementary ones.
A Boolean formula in Conjunctive Normal Form~(CNF) is a conjunction of different clauses.

Let $C$ be an arbitrary CNF formula, set $X$, $|X| = k$, be the set of variables occurring in $C$, and 
$\alpha \in \{0, 1\}^{|X|}$ be an arbitrary assignment of variables from $X$.
We define the \emph{interpretation} of formula $C$ on assignment $\alpha$ and \emph{substitution} of $\alpha$ into $C$ in a standard manner, see e.g.~\cite{ChangLee73}.
As a result of these operations, we define a Boolean function $f_C \colon \{0, 1\}^k \rightarrow \{0, 1\}$.
Assignment $\alpha \in \{0,1\}^k$ such that $f_C(\alpha) = 1$ is called a \emph{satisfying assignment} for $C$.
If a satisfying assignment exists for $C$, formula $C$ is called \emph{satisfiable}.
Otherwise, $C$ is called \emph{unsatisfiable}.

The problem SAT (or rather, CNF-SAT) in its decision variant is to determine for an arbitrary CNF formula $C$ if it is satisfiable.
In the search variant, it is additionally required to find a satisfying assignment of $C$ in the case if $C$ is satisfiable.
In the decision variant, SAT is NP-complete~\cite{Cook71}, and it is NP-hard in its search variant. 
However, as it was said above, modern SAT solvers demonstrated high efficiency in a number of practical application domains, successfully coping with formulas with dozens and hundreds of thousands of variables and clauses. 
Implementations of algorithms used to solve SAT are known as SAT solvers.

As it was mentioned above,  for many hard SAT instances, a complete SAT solver may work with an instance for a long time (say, a week) and there is no information about how long it will take to complete the work. 
The unpredictability of modern complete SAT solvers is associated by a number of researchers with the so-called \emph{heavy-tailed behavior} phenomenon~\cite{GomesSabh2021}. And in this context the following problem is very important and relevant: how hard is a specific CNF formula $C$ for a specific complete SAT solver $A$? 
In this sense, the following questions are fundamental: how to give a meaningful definition for a measure of hardness of CNF $C$ w.r.t. SAT solver $A$, and how to efficiently evaluate this measure in relation to specific $C$ and $A$? 
In this paper, we propose such a measure and build an efficient algorithm for its estimation.

\subsection{Some facts from probability theory and mathematical statistics}

A triple $\langle \Omega, \mathfrak{U}, Pr\rangle$, where $\Omega$ is a sample space, $\mathfrak{U}$ is a $\sigma$-algebra of events, and $\Pr \colon \mathfrak{U} \rightarrow [0, 1]$ is a probability function, is called a \emph{probability space} if Kolmogorov’s axioms for $\langle \Omega, \mathfrak{U}, \Pr\rangle$ hold, see e.g.~\cite{Feller71}. 
A \emph{random variable} is a total function $\xi \colon \Omega \rightarrow \mathbb{R}$. 
The set of values of $\xi$ is called the range or \emph{spectrum} of random variable $\xi$ and is denoted as $Spec(\xi)$.

We will work below only with finite sample spaces and with random variables which have finite spectra. 
In such cases, the corresponding $\sigma$-algebra is $\mathfrak{U} = 2^\Omega$ (power set of $\Omega$), and
the probability of an arbitrary event $a \in \mathfrak{U}$, $a = \{\omega_1^a, \ldots, \omega_q^a\}$, $\omega_i^a \in \Omega$, 
$i \in \{1, \ldots, q\}$ is defined as $\Pr[a] = \sum\limits_{r=1}^q \Pr[\omega_r^a]$. 
Let $Spec(\xi) = \{\xi_1, \ldots, \xi_s\}$ be the spectrum of $\xi$; 
with each $\xi_i$, $i \in \{1, \ldots, s\}$ let us connect the set of all preimages of $\xi_i$ w.r.t. the mapping $\xi \colon \Omega \rightarrow \mathbb{R}$, suppose that this set is an event from $\mathfrak{U}$, and let $p_i$ be the probability of this event. 
Thus, we connect with each $\xi_i \in Spec(\xi)$ some probability $p_i$, and it is clear (since $\xi \colon \Omega \rightarrow \mathbb{R}$ is total) that the following condition holds: $\sum\limits_{i=1}^s p_i = 1$. 
The set $\mathcal{D}(\xi) = \{p_1, \ldots, p_s\}$ is called the law of probability distribution (or just probability distribution) of random variable $\xi$.

The main characteristics of a random variable to be studied further are its first and second moments: \emph{expected value} (or expectation) and \emph{variance}, which are denoted as $\ev{\xi}$ and $Var(\xi)$ respectively. 
These parameters are defined as follows: $\ev{\xi} = \sum\limits_{i=1}^s \xi_i \cdot p_i$, 
$Var(\xi) = \ev{\xi^2} - \text{E}^2[\xi]$, where $\xi^2$ is a random variable which is associated with the same sample space as $\xi$, 
and accordingly has the probability distribution $\mathcal{D}(\xi^2) = \mathcal{D}(\xi)$, but the spectrum of $\xi^2$ looks as follows: $Spec(\xi^2) = \{(\xi_1)^2, \ldots, (\xi_s)^2\}$.

In our hardness estimations, we will further rely on the following inequality known as the Chebyshev’s (or Bienaym\'{e}-Chebyshev) inequality~\cite{Feller71}: for a random variable $\xi$ with finite expectation $\ev{\xi}$ and finite variance $Var(\xi)$, and for any $d > 0$, the following condition holds:
\begin{equation}
    \Pr\left[ |\xi - \ev{\xi}| \leq d \cdot \sqrt{Var(\xi)}\right] \ge 1 - \frac{1}{d^2}.
    \label{eq:cheb}
\end{equation}

In many practical applications, the space $\Omega$ is too large to exactly compute $Spec(\xi)$ and  $\ev{\xi}$. 
Usually, in such  situations, one can reason only w.r.t. some small (in comparison with $\Omega$) 
set of elementary events $R_\xi$, $R_\xi \subset \Omega$, associated with $N$ observed values of the random variable $\xi$: $\xi^1, \ldots, \xi^N$. 
In such cases (see e.g.~\cite{Wilks62}), statistical analogues of the quantities $\ev{\xi}$ and $Var(\xi)$ are used, called \emph{sample mean} (denoted as $\overline{\xi}$) and \emph{sample variance} (usually its unbiased form $s^2(\xi)$ is used) respectively:
\begin{equation}
    \overline{\xi} = \frac{1}{N} \cdot \sum\limits_{j = 1}^N \xi^j,
    \label{eq:sample-mean}
\end{equation}
\begin{equation}
    s^2(\xi) = \frac{1}{N - 1} \cdot \sum\limits_{j = 1}^N (\xi^j - \overline{\xi})^2.
    \label{eq:sample-variance}
\end{equation}

The idea to express some numerical parameter of an object under investigation via expected value of some random variable $\xi$ is the cornerstone 
of the Monte Carlo method~\cite{MU49}. 
As a rule, the sample mean $\overline{\xi}$ is used for this.
The connection between $\ev{\xi}$ and $\overline{\xi}$ is usually expressed with the help of some probabilistic inequalities (Chebyshev’s inequality, Hoeffding’s inequality~\cite{Hoeff63}, Chernoff’s bound~\cite{Chern52,MotwRagh95}), or the Central Limit Theorem~\cite{Feller71-vol1}. 
Such evaluations involve two numeric parameters $\varepsilon$ and $\delta$, $\varepsilon, \delta \in (0, 1)$. 
The parameter $\varepsilon$ is referred to as \emph{tolerance}, and $1 - \delta$ as the \emph{confidence level}. 
The corresponding relations are often called $(\varepsilon, \delta)$-approximations of $\ev{\xi}$, see e.g.~\cite{KarpLuby09}.

In more detail, let $\theta$, $\theta \geq 0$, be a parameter under consideration, and let $\tilde{\theta}$ be its probabilistic estimation.
Then, for given $\varepsilon, \delta \in (0, 1)$, the value $\tilde{\theta}$ is called an $(\varepsilon, \delta)$-approximation of $\theta$ if the following relation holds:
\begin{equation}
    \Pr\left[ (1 - \varepsilon) \cdot \theta \le \tilde{\theta} \le (1 + \varepsilon) \cdot \theta \right] \ge 1 - \delta,
\end{equation}
or, equivalently, $\Pr\left[|\tilde{\theta} - \theta| \le \varepsilon \cdot \theta \right] \ge 1 - \delta$.

\subsection{Backdoor hardness of CNF formulas}
The concept of a backdoor in relation to the Constraint Satisfaction Problem was introduced in~\cite{Williams03}.
More precisely, we will be interested in the Strong Backdoor Set~(SBS) concept from~\cite{Williams03} in the context of SAT.

Let $C$ be an arbitrary CNF formula over the set of variables $X$. 
Consider an arbitrary set $B \subseteq X$, and let $\{0,1\}^{|B|}$ be the set of all assignments of variables from $B$.
Following~\cite{Williams03}, for an arbitrary $\beta \in \{0,1\}^{|B|}$ denote $C[\beta/B]$ the CNF formula derived from $C$ by substitution of the assignment $\beta$ of variables $B$ and consequent simplification of the resulting formula.
Let $P$ be some polynomial-time algorithm to which, as in~\cite{Williams03}, we will refer as to a \emph{sub-solver}.

\begin{definition}[\cite{Williams03}]
Set $B \subseteq X$ is called a \emph{strong backdoor set}~(SBS) for $C$ w.r.t.\ a polynomial-time algorithm (sub-solver) $P$, if for any 
$\beta \in \{0,1\}^{|B|}$ the algorithm $P$ outputs the solution of SAT for CNF formula $C[\beta/B]$.
\end{definition}

Thus, if there exists some SBS for CNF $C$, then SAT for $C$ is solved in time $poly(|C|)\cdot 2^{|B|}$, where $poly(\cdot)$ is some polynomial, $|C|$ is the length of the binary description of $C$.
There are many examples of formulas $C$ with an SBS, the number of variables in which may not exceed a few percent (or even fractions of a percent) of the total number of variables in $C$. 
Such, for example, are CNF encodings of verification problems or cryptanalysis problems. 
In these cases, the set of variables encoding the input of the cryptographic function under consideration forms the SBS 
w.r.t. the simplest Boolean constraint propagation rule, the Unit Propagation rule (UP)~\cite{Dowling84}. 
Thus, for example, a CNF encoding the problem of finding the preimage of the SHA-256 cryptographic function, that is, the problem of inverting the function $f_{\text{SHA-256}} \colon \{0, 1\}^{512} \rightarrow \{0, 1\}^{256}$, has an SBS w.r.t. the UP rule, which consists of 512 variables, while the total number of variables in this formula is $49098$.
It is easy to see that an SBS of size significantly smaller than 512, if it existed, would give a non-trivial attack on this function.

The article~\cite{Williams03} established the following fact.
\begin{theorem}[\cite{Williams03}]
\label{th:williams}
If a CNF formula $C$ has an SBS $B$ of size $|B| < |X| / 2$, then there exists an algorithm that solves SAT for $C$ in time
\begin{equation}
    poly(|C|) \cdot \left(\frac{2|X|}{\sqrt{|B|}}\right)^{|B|}.
\end{equation}
\end{theorem}

\begin{corollary}[\cite{Williams03}]
For CNF formulas with an SBS of size smaller than $\frac{k}{4.404}$, $k=|X|$, SAT can be solved in a deterministic way in time $O(poly(|C|)\cdot c^k)$, 
where $c < 2$.
\end{corollary}

The algorithm from~\cite{Williams03} (hereinafter called the \enquote{WGS algorithm}), which is implied in the statement of Theorem~\ref{th:williams}, is very simple. 
For a particular polynomial sub-solver $P$, all sets consisting of one variable (their number is $k = |X|$) are first checked for the property of being an SBS, then all sets of two variables (their number is ${k \choose 2}$), and so on. 
Obviously, the first SBS found by the algorithm is the SBS of the smallest cardinality, further referred to as the \emph{minimum SBS}. 
However, it is easy to see that the algorithm is extremely inefficient for finding the minimum (or even minimal) SBS.

\begin{corollary}
The worst-case complexity estimation of the WGS algorithm in application to the search for a minimal SBS of formula $C$
over the set of variables $X$ is $O(p(|C|) \cdot 3^{|X|})$ for some polynomial $p(\cdot)$.
\end{corollary}
\begin{proof}
Obviously, the worst case is when the algorithm goes through all subsets of the set $X$ and runs the algorithm $P$ on formulas $C[\beta/X]$ for all $\beta \in \{0, 1\}^{|X|}$. 
In this case, the complexity of the WGS algorithm is limited by the product of some polynomial $p(|C|)$ and the following value:
\begin{equation*}
    \sum\limits_{i=1}^k {k \choose i} \cdot 2^i = \left(\sum_{i = 0}^k {k \choose i} \cdot 2^i\right) - 1 = 3^k - 1,
\end{equation*}
where $k = |X|$. 
Thus, the corollary is proved.
\end{proof}

The question of how hard the formula is for some complete SAT solving algorithm is of fundamental importance. 
As it was said above, there is a large number of papers in which hardness estimates w.r.t. specific algorithms were constructed for specific families of formulas.
Mostly, these works dealt with analytical estimates. 
The paper~\cite{Ansoteg08} undertook a systematic analysis of various metrics of hardness and explored their relationship. 
Although~\cite{Ansoteg08} focuses on a few tree-like metrics, of importance to us are the conclusions made in this paper about the possibility to assess the hardness of a CNF formula using an SBS.

In this work, as in many papers on proof complexity and hardness of SAT \cite{Alekhn02,Ansoteg08,CookRekh79,Tse70}, we consider only unsatisfiable formulas.
This choice is due to the fact that for satisfiable formulas (especially if they have many satisfying assignments) a SAT solver may often get ``lucky'' and quickly find a satisfying assignment (satisfiability certificate).
This is, however, impossible for unsatisfiable formulas.

Analysis of the results from~\cite{Ansoteg08} suggests the following definition. 
In fact, it uses an SBS to evaluate the hardness of an arbitrary CNF formula and reduces this problem to an optimization problem.

\begin{definition}[b-hardness]
Let $C$ be an arbitrary unsatisfiable CNF formula over the set of variables $X$, and let $B$, $B \in 2^X$, be an arbitrary SBS for $C$ w.r.t.\ a polynomial-time algorithm $P$.
Let $t_P(C[\beta/B])$ be the runtime of algorithm $P$ on the CNF formula $C[\beta/B]$ for an arbitrary $\beta \in \{0, 1\}^{|B|}$.
Consider the value $\mu_{B, P}(C) = \sum_{\beta \in \{0, 1\}^{|B|}} t_P(C[\beta/B])$.
The \emph{backdoor hardness} (or \emph{b-hardness}) of formula $C$ w.r.t. algorithm $P$ is specified as $\mu_P(C) = \min\limits_{B \in 2^X: \; B \text{\; is an SBS}} \mu_{P, B}(C)$, where the minimum is taken among all possible SBSes for $C$ w.r.t. $P$.
\label{def:sbs-definition}
\end{definition}

\section{Our Motivation}
\label{sec:motivation}
As we can see from all of the above and, in particular, from Definition~\ref{def:sbs-definition}, the backdoor hardness $\mu_P(C)$ of formula $C$ w.r.t. a polynomial algorithm $P$ should be close to 
$\mu_{B^*,P}(C)$, where $B^*$ is the minimum SBS. 
Thus, if there were an efficient algorithm to find a minimum SBS, it would be of practical interest. 
Unfortunately, the WGS algorithm, as we have seen, is not efficient, and at first glance it is not clear how the minimal SBS search algorithm, which is significantly more efficient than the WGS algorithm, might look like. 
And the main reason for this is not in the enumeration of various alternatives of $B \in 2^X$, which can be done e.g. using metaheuristics~\cite{Luke15}, 
but in the fact that for each $B \in 2^X$ we are forced to check that $C[\beta/B]$ is solved by algorithm $P$ for all $\beta \in \{0, 1\}^{|B|}$.

So, we would like to iterate over different sets $B \in 2^X$ using some smart search heuristic (here, we use a metaheuristic optimization algorithm), and at the same time evaluate the hardness of $C$
w.r.t. set $B$ using some efficiently calculated measure.
Such an estimate should be built as quickly as possible so that the optimization algorithm has time to go through as many candidate solutions as possible in a certain search space, and do it in reasonable time. The basic idea of the estimations used below is to evaluate the usefulness of $B$ by evaluating the characteristics of some random variable using random sampling,
and giving reasonable accuracy guarantees for the resulting estimations. 
That is, in fact, we are talking about the estimates obtained using the Monte Carlo method~\cite{MU49}.

In more detail, following~\cite{CP2021}, we further use the \textit{decomposition hardness} (d-hardness) concept, which is different from b-hardness in that it uses a complete SAT solver (instead of a polynomial sub-solver) to solve SAT
for instances $C[\beta/B]$. 
Since in this situation SAT for each $C[\beta/B]$ is solved in finite time, we can correctly associate with $B$ a special random variable, such that the decomposition hardness will be defined via the expected value of this random variable.

\section{Decomposition Hardness of CNF Formulas}
\label{sec:d-hardness}
Let $A$ be some deterministic and complete SAT solver, e.g. based on the Conflict Driven Clause Learning (CDCL) algorithm~\cite{MSS96,MS_Handb09}. 
Since $A$ is complete, then its runtime on an arbitrary formula has a finite value.
The runtime of $A$ can be measured in any appropriate units. 
In the case of CDCL-based SAT solvers, this can be the execution time in seconds, the number of conflicts, or the number of times the Unit Propagation rule was applied. 

Let us define the decomposition hardness of an unsatisfiable formula $C$ over $X$ w.r.t. a complete SAT solver $A$
in the following manner.

\begin{definition}[d-hardness]
Let $C$ be an arbitrary CNF formula over the set of variables $X$, and consider an arbitrary set $B$, $B \in 2^{X}$. 
Let $t_A(C[\beta/B])$ be the runtime of a deterministic and complete SAT solver $A$ on the formula $C[\beta/B]$ for an arbitrary $\beta \in \{0, 1\}^{|B|}$.
The decomposition hardness (d-hardness) of $C$ w.r.t. $B$ and $A$ is specified as $\mu_{B,A}(C) = \sum_{\beta \in \{0,1\}^{|B|} } t_A(C[\beta /B])$.
The d-hardness of $C$ w.r.t. $A$ is defined as $\mu_A(C) = \min_{B \in 2^X} \mu_{B,A}(C)$.
\label{def:d-hardness}
\end{definition}

Next, we will establish that the d-hardness of $C$ w.r.t. $A$ and $B$ can be expressed using a special random variable.
Consider the following theorem.

\begin{theorem}
Let $C$ be an arbitrary CNF formula over the set of variables $X$, and let $A$ be some deterministic and complete SAT solving algorithm. 
Then, for an arbitrary set $B$, $B \in 2^X$, there exists a random variable $\xi_B$ with finite spectrum, expectation, and variance, such that the following fact holds:
{\upshape
\begin{equation}
\label{eq:dh-random-variable}
\mu_{B,A}(C) = 2^{|B|} \cdot \text{E}[\xi_B].
\end{equation}
}
\label{th:dh-random-variable-theorem}
\end{theorem}
\begin{proof}
The random variable $\xi_B$ mentioned in the theorem is defined in the following manner. 
Define on $\Omega = \{0, 1\}^{|B|}$ a uniform distribution, and connect with each $\beta \in \{0, 1\}^{|B|}$ the value $t_A(C[\beta/B])$ 
(expressed, e.g., in the corresponding number of unit propagations). 
Since $A$ is deterministic and complete, the following mapping is specified: $\xi_B \colon \Omega \rightarrow \mathbb{N}$. 
Let $Spec(\xi_B) = \{\xi_1, \ldots, \xi_s\}$ be the spectrum of $\xi_B$. 
And let us specify the probability distribution of $\xi_B$ as follows: $\mathcal{D}(\xi_B) = \{p_1, \ldots, p_s\}$, $p_i = \frac{\# \xi_i}{2^{|B|}}$, 
where by $\# \xi_i$, $i \in \{1, \ldots, s\}$, we denote the number of such $\beta \in \{0, 1\}^{|B|}$ that $t_A(C[\beta/B]) = \xi_i$. 
It is clear that the following facts hold:
\begin{equation*}
    \sum\limits_{\beta \in \{0, 1\}^{|B|}} t_A(C[\beta/B]) = \sum\limits_{i=1}^s \xi_i \cdot \#\xi_i = 2^{|B|} \cdot \sum\limits_{i = 1}^s \xi_i \cdot \frac{\#\xi_i}{2^{|B|}} = 2^{|B|} \cdot \ev{\xi_B},
\end{equation*}
and, thus, the equality~\eqref{eq:dh-random-variable} holds.
\end{proof}

The relation~\eqref{eq:dh-random-variable} means that the value $\mu_{B,A}(C)$ can be evaluated by estimating $\ev{\xi_B}$ using the Monte Carlo method. 
In more detail, we would like to construct some value $\tilde{\mu}_{B,A}(C)$ which is an $(\varepsilon, \delta)$-approximation of $\mu_{B,A}(C)$, 
i.e. for any fixed $\varepsilon, \delta \in (0, 1)$ the following condition holds:
\begin{equation}
    \Pr\left[(1 - \varepsilon) \cdot \mu_{B,A}(C) \leq \tilde{\mu}_{B,A}(C) \leq (1 + \varepsilon) \cdot \mu_{B,A}(C)\right] \geq 1 - \delta.    
\end{equation}

Let us fix some natural number $N$.
Given $C$, $B$, and $A$, let us carry out $N$ independent observations of random variable $\xi_B$ introduced above. 
We may consider these $N$ observations as one observation of $N$ independent random variables with the same probability distribution (recall that we assume $A$ to be deterministic).
Denote the corresponding values as $\xi^1, \ldots, \xi^N$.
Define $\tilde{\mu}_{B,A}(C)$ as:
\begin{equation}
    \tilde{\mu}_{B,A}(C) = \frac{2^{|B|}}{N} \cdot \sum\limits_{j=1}^{N} \xi^j.
    \label{eq:tilde-mu}
\end{equation}
Let us establish the following fact.

\begin{theorem}
\label{th:tilde-mu-N}
Let $C$ be an arbitrary CNF formula over variables $X$, $A$ be a deterministic complete SAT solving algorithm, and $B$ be an arbitrary subset of $X$. 
Then for $\tilde{\mu}_{B,A}(C)$ specified by~(\ref{eq:tilde-mu}) and for any $\varepsilon > 0$, $\delta > 0$, the condition~(\ref{eq:dh-random-variable}) 
holds for any {\upshape $N \geq \frac{Var(\xi_B)}{\varepsilon^2 \cdot \delta \cdot \text{E}^2[\xi_B]}$.}
\end{theorem}
\begin{proof}
Due to the assumptions on $A$, the random variable $\xi_B$ has a finite expected value $\ev{\xi_B}$ and finite variance $Var(\xi_B)$. 
If $Var(\xi_B)=0$, then $Spec(\xi_B)=\{a\}$, where $a$ is some constant: $a>0$. 
In this case, the claim of the theorem is trivially satisfied. 
Below, let us assume that $Var(\xi_B)>0$. 
Next, we use the Chebyshev's inequality~\eqref{eq:cheb}: $\Pr \left[ |\zeta - \ev{\zeta}|\leq d \cdot \sqrt{Var(\zeta)} \right] \geq 1-\frac{1}{d^2}$
for some random variable $\zeta$ which satisfies the required assumptions.

Fix an arbitrary $\varepsilon \in (0, 1)$, select $d$ such that $d\cdot\sqrt{Var(\zeta)} = \varepsilon \cdot \ev{\zeta}$,
and transform the latter inequality to the following form:
\begin{equation}
    \label{eq:TA-eq3}
    \Pr \left[ |\zeta - \ev{\zeta}|\leq \varepsilon \cdot \ev{\zeta}\right] \geq 1-\frac{Var(\zeta)}{\varepsilon^2\cdot \text{E}^2[\zeta]}.
\end{equation}
Considering $N$ independent observations of $\xi_B$ as a single observation of $N$ independent random variables with the same distribution,
we can conclude that the following holds: $\ev{\xi^1} = \ldots = \ev{\xi^N} = \ev{\xi_B}$, $Var(\xi^1) = \ldots = Var(\xi^N) = Var(\xi_B)$. 
Consider the random variable $\zeta = \sum_{j=1}^N \xi^j$. 
If we apply inequality~\eqref{eq:TA-eq3} to it, we get:
\begin{equation*}
\Pr\left[ (1 - \varepsilon) \cdot \text{E}\left[\sum\limits_{j = 1}^N \xi^j\right] \leq \sum\limits_{j = 1}^N \xi^j \leq (1 + \varepsilon) \cdot \text{E}\left[\sum\limits_{j = 1}^N \xi^j\right]\right] \geq 1 - \frac{Var\left(\sum\limits_{j = 1}^N \xi^j\right)}{\varepsilon^2 \cdot \text{E}^2\left[\sum\limits_{j = 1}^N \xi^j\right]}.
\end{equation*}
After simple steps, the latter inequality is transformed into the following form:
\begin{equation*}
    \Pr\left[(1 - \varepsilon) \cdot \ev{\xi_B} \leq \frac{1}{N} \cdot \sum\limits_{j = 1}^N \xi^j \leq (1 + \varepsilon) \cdot \ev{\xi_B}\right] \geq 1 - \frac{Var(\xi_B)}{\varepsilon^2 \cdot N \cdot \text{E}^2[\xi_B]}.
\end{equation*}
Finally, taking into account~\eqref{eq:dh-random-variable} and~\eqref{eq:tilde-mu}, we have:
\begin{align}
    \Pr\left[(1 - \varepsilon) \cdot \mu_{B,A}(C) \leq \tilde{\mu}_{B,A}(C) \leq (1 + \varepsilon) \cdot \mu_{B,A}(C)\right] \geq 1 - \frac{Var(\xi_B)}{\varepsilon^2 \cdot N \cdot \text{E}^2[\xi_B]}.
    \label{eq:th-tilde-mu-N-last}
\end{align}
Let us fix some $\delta \in (0, 1)$. 
Keeping in mind~\eqref{eq:th-tilde-mu-N-last}, we can conclude that~\eqref{eq:tilde-mu} will hold for any $N \geq \frac{Var(\xi_B)}{\varepsilon^2 \cdot \delta \cdot \text{E}^2[\xi_B]}$. 
And, thus, the theorem is proved.
\end{proof}

\section{Evaluation of d-hardness via Pseudo-Boolean \\Optimization}
\label{sec:d-hardness-search}

Returning to what was said above, we intend to reduce the problem of estimating d-hardness to the problem of minimizing some function whose values are calculated using the Monte Carlo method. 
Actually, we want to find some set $B \in 2^X$, on which the value of the estimated function suits
us: ideally, we want to find a set $B$ on which this function reaches a minimum value.

As it is customary in metaheuristic optimization, we will use the term \emph{fitness function} (see e.g.~\cite{Luke15}) for the function whose minimum we are looking for. 
Note that an arbitrary set $B \in 2^X$ can be represented by a Boolean vector of length $k = |X|$: ones in such a vector will indicate which variables from $X$ are present in $B$, 
and zeros will indicate the variables that are absent in $B$. 
Thus, our fitness function is a pseudo-Boolean function~\cite{Boros02} of the form $F \colon \{0, 1\}^{|X|} \rightarrow \mathbb{R}$.

In the situation with estimating the d-hardness of formula $C$ w.r.t. SAT solver $A$ and set $B$, the fitness function has the form:
\begin{equation}
    F_{C,A,N}(\lambda_B) = \frac{2^{|B|}}{N} \cdot \sum\limits_{j=1}^N t_A(C[\beta_j/B]),
    \label{eq:dh-fun}
\end{equation}
where $\lambda_B \in \{0, 1\}^{|X}$ is a Boolean vector which defines the set $B$, vectors $\beta_1, \ldots, \beta_N$ 
from $\{0, 1\}^{|B|}$ form a random sample of size $N$, and $\xi^j = t_A(C[\beta_j/B])$, $j \in \{1, \ldots, N\}$, are the observed values of
random variable $\xi_B$ in $N$ independent experiments.

We would like the value of the function~\eqref{eq:dh-fun} to give an $(\varepsilon, \delta)$-approximation of value $\ev{\xi_B}$ for fixed $\varepsilon$ and $\delta$. 
It follows from Theorem~\ref{th:dh-random-variable-theorem} that for this it suffices to take any $N$:
\begin{equation}
    N \ge \frac{Var(\xi_B)}{\varepsilon^2 \cdot \delta \cdot \text{E}^2[\xi_B]}.
    \label{eq:dh-N-condition}
\end{equation}

Note that, unfortunately, in the general case we cannot say anything about $Var(\xi_B)$. 
Formally, this quantity is finite (since algorithm $A$ is complete). 
However, in practice it can be very large, and this fact is directly related to the heavy-tailed behavior phenomenon~\cite{GomesSabh2021}.

In computational experiments, when using specific values of the parameters $\varepsilon$ and $\delta$, 
we are forced to use the statistical estimates of $\ev{\xi_B}$ and $Var(\xi_B)$ given by formulas~\eqref{eq:sample-mean} and~\eqref{eq:sample-variance}.
Of course, in this case, unfortunately, we cannot speak about guarantees of the accuracy of the estimates obtained (for relatively small $N$), 
but this situation is in principle typical of many statistical experiments.

Taking into account all said above, when we compute the value of function~\eqref{eq:dh-fun} for a particular set $B \in 2^X$, 
we use the following statistical analogue of formula~\eqref{eq:dh-N-condition}:
\begin{equation}
    N \ge \frac{s^2(\xi_B)}{\varepsilon^2 \cdot \delta \cdot (\overline{\xi}_B)^2}.
    \label{eq:dh-N-condition-stat}
\end{equation}
More concretely, at the starting point we choose some relatively small value of $N$ (say, $N=1000$), construct a random sample and calculate $\overline{\xi}_B$ and $s^2(\xi_B)$
using formulas~\eqref{eq:sample-mean} and~\eqref{eq:sample-variance}.
For fixed $\varepsilon,\delta \in (0, 1)$ we check if inequality~\eqref{eq:dh-N-condition-stat} is satisfied.

If it is, we conclude that our estimation of $\mu_{B,A}(C)$ has the required accuracy.
If not, we augment our random sample with $N$ new observations of $\xi_{B}$, thus doubling the random sample size; after this, we recalculate $\overline{\xi}_B$ and $s^2(\xi_B)$. 
These steps are repeated until condition~\eqref{eq:dh-N-condition-stat} is satisfied.

It is important to note that, in general, the value of $\xi_B$ and the fitness function~\eqref{eq:dh-fun} cannot be calculated efficiently.
Indeed, for some small set $B$, the time needed to compute these values may well be comparable with the time needed to solve SAT for the original CNF formula $C$.
However, for any formula $C$ there exists a set $B$ such that the value of $\xi_B$ for any $\beta \in \{0,1\}^{|B|}$ is calculated efficiently.
The simplest and most obvious example is $B = X$.
Another case is when $B$ is a Strong Unit Propagation Backdoor Set~(SUPBS): an SBS w.r.t. the unit propagation rule~\cite{Williams03}; many SAT instances from such domains as, for example, bounded model checking and cryptanalysis, have small SUPBSes.

For practical problems with large $X$ (e.g. thousands of variables), efficient optimization of function~\eqref{eq:dh-fun} can be organized not on the entire set $2^X$, 
but only on $2^{B_0}$ for some $B_0 \subset X$, such that the function~\eqref{eq:dh-fun} is computed efficiently on $B_0$.
In our experiments, the set $B_0$ is built using efficient heuristics similar to the ones used in look-ahead SAT solvers~\cite{Heule09}. 
More details on this matter are given in Section~\ref{sec:search-space-reduction}. 


The objective function~\eqref{eq:dh-fun} does not have an analytical definition, so the only viable option for its optimization is the application of some black-box optimization algorithms, e.g. metaheuristic algorithms~\cite{Luke15}.
In this work, we use a variation of a genetic algorithm previously used in~\cite{evostar}.
The genetic algorithm operates with a population $\Pi_{\text{cur}}$ of vectors $\lambda_B$, and on each iteration the algorithm prepares a new population $\Pi_{\text{new}}$ called the offspring, ensuring that $|\Pi_{\text{cur}}| = |\Pi_{\text{new}}| = R$ for some fixed number $R$.
Population $\Pi_{\text{cur}}$ is associated with a distribution $\mathcal{D}_{\text{cur}} = \{p_1, \ldots, p_R\}$, where
\begin{equation*}
    p_i = \frac{1 / F_{A,C,N}(\lambda_{B_i})}{\sum\limits_{j=1}^R \left(1 / F_{A,C,N}(\lambda_{B_j})\right)},  i \in \{1,\ldots,R\}.
\end{equation*}

The new population $\Pi_{\text{new}}$ is constructed in the following way.
First, $E$ individuals from $\Pi_{\text{cur}}$ with the highest fitness function values are added to $\Pi_{\text{new}}$; this corresponds to the use of an elitist genetic algorithm~\cite{Luke15}.
Second, $G$ individuals are added in the following way:
select a pair of individuals from $\Pi_{\text{cur}}$, choosing each individual independently w.r.t the distribution $\mathcal{D}_{\text{cur}}$, apply to the selected pair the standard two-point crossover~\cite{Luke15}, add the resulting individuals to $\Pi_{\text{new}}$, and repeat until $G$ individuals are added.
Third, $H$ individuals are added in the following way:
select an individual from $\Pi_{\text{cur}}$ w.r.t. the distribution $\mathcal{D}_{\text{cur}}$, apply to it the $(1+1)$ heavy-tailed mutation operator from~\cite{Doerr17} with parameter $\beta = 3$, and repeat until $H$ individuals are added.
The procedure ensures that $R = E + G + H$.

\section{Search Space Reduction}
\label{sec:search-space-reduction}

Note that when using the method proposed in~\cite{CP2021} for some CNF formula $C$ over $X$ to seek some decomposition set $B$ as a subset of $X$, we can hope that our search will be successful only if the cardinality of $X$ does not exceed several hundred variables. 
This is the consequence of the fact that the function~\eqref{eq:dh-fun} is computationally costly to compute in the general case. 
As we said above, in many practical cases we can narrow down the search space (on which the function~\eqref{eq:dh-fun} is optimized) by studying the features of the original combinatorial problem which was reduced to the SAT instance under consideration: e.g., for SAT instances originating from verification or cryptanalysis problems, we can limit our search space to some SUPBS of the considered formula.

We can also build a reduced search space by utilizing well-known SAT solving techniques which are used for determining sets of promising (important, interesting) variables. 
First of all, we mean here the look-ahead heuristics used in look-ahead SAT solvers~\cite{Heule09} and Cube-and-Conquer strategy~\cite{CC2012}. 
In accordance with these heuristics, for each variable from $X$, some measure of its usefulness is calculated, which is referred to as a \emph{reduction measure}. 
In more detail, in the first look-ahead heuristics, the following approach was used: for an arbitrary variable $x \in X$ and for each literal from $\{x, \lnot x\}$, the number of so-called \emph{new clauses} was calculated; that is, such clauses that did not exist before in the original formula.

Each new clause is the result of stripping literals from the original clauses, if we consider formulas $x \wedge C$ and $\lnot x \wedge C$ and apply to the latter unit propagation and pure literal rules.
Thus, each literal from $\{x, \lnot x\}$ is assigned a natural number, which is the number of new clauses generated (in the sense described above) by this literal. 
The reduction measure of variable $x$ is usually defined as the product of these numbers, but in some cases their sum can be used (see details in~\cite{Heule09}).

It is easy to see that the time complexity of the algorithm for calculating reduction measures for all variables from $X$ is polynomial in $|C|$. 
Suppose that we have calculated the reduction measures for all variables from $X$. 
Let us sort in descending order the reduction measures of its elements,
take the first $m$ variables w.r.t. this order, and denote the resulting set by $B_0$. 
We select $m$ such that the value the function~\eqref{eq:dh-fun} for $B_0$ would be computed efficiently. 
If $B_0$ is such a set, then $2^{B_0}$ is the reduced search space for optimization of function~\eqref{eq:dh-fun}.

One of the results of our computational experiments is the fact that often comparable and even better results compared to standard look-ahead formation strategies can be achieved using the simpler heuristic described below, which can, therefore, be assigned to the same class of heuristics. 
For each variable $x_i \in X$ we first substitute the positive literal of $x_i$ into $C$ (using Unit Propagation) and calculate the total number of variables $w_i^+$ such that their values (literals) are derived by UP application. 
Then we repeat the procedure for the negative literal $\lnot x_i$, and calculate $w_i^-$ in the same way. 
Let us define the weight of the variable in the formula as $w_i = w_i^+ + w_i^-$. 
The initial decomposition set $B_0$, $B_0 \subset X$, is chosen as the set of $m$ variables with the largest values of $w_i$, also ensuring that the function~\eqref{eq:dh-fun} is efficiently calculated for~$B_0$.

\section{Additional Techniques: Accounting for Simple Subproblems and Using Several Backdoors}
\label{sec:additional}

The results of this section are inspired by the paper~\cite{AAAI2022}, in which the probabilistic approach to estimate the fraction of ``simple'', polynomially decidable subproblems $C[\beta/B]$ (among all possible such problems for some set $B$) was presented. 
More specifically, let $P$ be some polynomial sub-solver in the sense of~\cite{Williams03}. We will use the notation $C[\beta /B]\in S(P)$ if SAT for formula $C[\beta /B]$ is decided by $P$ and will write $C[\beta /B]\notin S(P)$ in the opposite case. 
Denote as $\rho$ the fraction of such vectors $\beta \in \{0,1\}^{|B|}$ for which $C[\beta /B]\in S(P)$ for some fixed set $B \subseteq X$. 
In more detail, let us use the following definition.

\begin{definition}[$\rho$-backdoor]
Let $C$ be some CNF formula over the set of variables $X$, and $P$ be some polynomial sub-solver. Then an arbitrary set $B$, $B \subseteq X$, is called a $\rho$-backdoor for $C$ w.r.t. $P$ if the fact $C[\beta /B]\in S(P)$ holds for at least $\rho \cdot 2^{|B|}$ vectors $\beta \in \{0,1\}^{|B|}$.
\label{def:rho-backdoor}
\end{definition}

Recall that when we search for a set $B$ associated with a minimal d-hardness value, we have several options of how to start the search: from the entire set $X$, from a SUPBS $B_{0} \subset X$, or its heuristically constructed analogue as described in the previous section.
Therefore, by construction of these initial sets, at least on the initial stages of the search, the subproblems $C[\beta/B]$ are solved efficiently.
Experiments showed that, apart from that, most of these subproblems are polynomially solvable: i.e. $B$ is some $\rho$-backdoor with $\rho$ which is quite close to~1. 
This observation allows for a simple, but important modification of the d-hardness objective function implementation.

Let us recall that in order for the measurements of the random variable $\xi_{B}$ to be independent,
we are forced to launch the SAT solver from scratch to measure its running time for each suproblem $C[\beta/B]$.
But if many subproblems are simple, we can keep an additional SAT solver (restricted to the use of unit propagation only) that will be used to check if $C[\beta/B] \in S(P)$. 
This solver can be used incrementally through the assumptions mechanism, and the resulting 
workload measures of this polynomial solver (time used for propagation, number of unit propagations, etc.) can be used as corresponding measurements of random variable $\xi_{B}$.


The concept of $\rho$-backdoors from~\cite{AAAI2022} is also interesting in the sense that we can use several such sets to solve a hard SAT instance. 
Indeed, let $C$ be some hard SAT instance and let $B_1, \ldots, B_s$ be different $\rho$-backdoors.
Thus, for any $i \in \{1, \ldots, s\}$ we have that for $\rho_i \cdot 2^{|B_i|}$ formulas $C[\beta/B_i]$ the following holds: $C[\beta/B_i] \in S(P)$ for some polynomial sub-solver $P$. 
For each $i \in \{1, \ldots, s\}$ denote by $\Gamma_i$ the set of all $\beta \in \{0, 1\}^{|B_i|}$ such that $C[\beta/B_i] \notin S(P)$, 
and consider the Cartesian product: $\Gamma = \Gamma_1 \times \cdots \times \Gamma_s$. 
Denote as $C[\gamma]$ the CNF formula obtained by substitution of an arbitrary $\gamma \in \Gamma$ to formula $C$. 
For each $\gamma \in \Gamma$ let us apply a complete SAT solver to formula $C[\gamma]$. 
Note that $C$ is unsatisfiable if and only if all formulas $C[\beta/B_i] \in S(P)$ for all $i \in \{1, \ldots, s\}$ are unsatisfiable,
and all formulas $C[\gamma]$ for all $\gamma \in \Gamma$ are unsatisfiable.

In the context of the above, let us emphasize that in order to prove the unsatisfiability of $C$, one can invoke some algorithm with polynomial runtime 
$\sum_{i = 1}^s \rho_i \cdot 2^{|B_i|}$ times to solve the ``simple'' subproblems, and apply a complete SAT solver to
$2^{\sum_{i=1}^s |B_i|} \cdot \Pi_{i = 1}^s \left (1-\rho_{i} \right)$ 
formulas $C[\gamma]$, $\gamma \in \Gamma$. 
It should be noted that for each $\gamma \in \Gamma$, the formula $C[\gamma]$ is obtained as a result of substitution of $\sum_{i=1}^s |B_i|$ bits into $C$, 
and such a substitution can significantly simplify the original SAT instance $C$. 
If for each $i \in \{1, \ldots, s\}$ the value $\rho_i$ is close to 1, then the time to prove the unsatisfiability of $C$ using $\rho$-backdoors $B_1, \ldots, B_s$, can be significantly lower than the runtime of SAT solver $A$ in application to the original formula $C$. 
Also, in practice we see that the vast majority of formulas $C[\gamma]$ are polynomially solvable.

\section{Unsatisfiability Proofs and Backdoors}
\label{sec:proofs}

Most CDCL SAT solvers have an option to produce an unsatisfiability proof, apart from the actual solving of a SAT formula. 
An unsatisfiability proof is produced on the basis of conflict clauses learned by the SAT solver~\cite{Goldberg03}. 
A satisfiability certificate, i.e. a satisfying assignment, can be checked in polynomial time. 
On the contrary, checking an unsatisfiability proof is a much more computationally expensive procedure, since the size of the proof can be exponential in the size of the formula. 
In many cases, checking a proof may even take more time than the actual SAT solving. 
Therefore, a way to reduce unsatisfiability proof checking time is needed.

It is worth to note that the decomposition sets generated by the approach proposed in this paper provide a natural way to facilitate faster unsatisfiability proof checking. 
Indeed, a decomposition set $B$ provides a partitioning of the initial CNF formula $C$ into $2^{|B|}$ smaller formulas $C[\beta/B]$.
If each of these weakened formulas is solved independently, then $2^{|B|}$ unsatisfiability proofs $\mathcal{P}[\beta/B]$
for the corresponding formulas $C[\beta/B]$ are generated. 
These proofs may be checked independently and in parallel. 
For checking the proofs we used the well-known tool DRAT-trim\footnote{\url{https://github.com/marijnheule/drat-trim}}~\cite{drat-trim-1,drat-trim-2}.

Let us also note that using a set $B$ for speeding up unsatisfiability proof checking looks especially attractive in the case when $B$ is some $\rho$-backdoor with $\rho$ close to 1. 
Indeed, in this case we can split the proof verification for formulas $C[\beta/B]$, $\beta \in \{0, 1\}^{|B|}$, into the following stages:
\begin{itemize}
    \item verify proofs for the $(1 - \rho) \cdot 2^{|B|}$ formulas $C[\beta/B]$: $C[\beta/B] \notin S(P)$;
    \item verify proofs for the remaining $\rho \cdot 2^{|B|}$ formulas $C[\beta/B]$: $C[\beta/B] \in S(P)$. 
\end{itemize}

In the case of formulas of the first type, we separately generate and store the proof $\mathcal{P}[\beta/B]$ for each such formula. 
For realistic practical sizes of set $B$, say, $|B| < 15$, the number of such ``subproofs'' typically does not exceed 1000.

Since the unsatisfiability of any formula of the second type is proved only by UP, they all have short proofs. 
However, the number of these can be quite significant due to the fact that $\rho$ is close to 1.
We propose the following approach to produce, store, and check proofs for simple subformulas $C[\beta/B] \in S(P)$.
The idea is to split the set of assignments that lead to simple subformulas $C[\beta/B]$ into $K$ ($K \ll 2^{|B|}$) subsets $Q_1, \ldots, Q_K$, $Q_k = \{\beta^1_k, \ldots, \beta^{d_k}_k\}$, $k \in \{1, \ldots, K\}$.
For all assignments from the subset $Q_k$ we construct a CNF formula that is satisfiable if and 
only if (1) the original CNF formula $C$ is satisfiable and (2) values of variables from $B$ are limited to combinations defined by $Q_k$.
Obviously, we can connect with each $\beta \in \{0, 1\}^{|B|}$ some cube $\sigma(\beta)$; recall (see e.g.~\cite{CC2012}) that a cube is formula which is dual to a clause and, thus, it is conjunction of some literals.
Let $\sigma_1^k, \ldots, \sigma^k_{r_k}$ be all cubes corresponding to assignments from $Q_k$ ($|Q_k| = r_k$).
Let us introduce new Boolean variables $u_1, \ldots, u_{r_k}$ and consider the following formula:
\begin{equation}
    C \wedge (u_1 \equiv \sigma_1^k) \wedge \cdots \wedge (u_{r_k} \equiv \sigma^k_{r_k}) \wedge (u_1 \vee \cdots \vee u_{r_k}).
    \label{eq:cubes}
\end{equation}

It is easy to understand that the formula~\eqref{eq:cubes} is unsatisfiable if and only if all formulas $C[\beta/B]$, $\beta \in Q_k$ are unsatisfiable. 
Using Tseitin transformations~\cite{Tse70}, we switch from SAT for the formula~\eqref{eq:cubes} to SAT for some CNF formula $C_k$.

In general, for each formula and backdoor one may select an optimal value of parameter $K$ so that the total proof checking time would be minimal.
A typical case we observed in our experiments is illustrated in Fig.~\ref{fig:check-easy-g3} for SAT solver Glucose~3, one of the SAT instances used in~\cite{CP2021} ($PvS_{7, 4}$), and ten different backdoors found with one of the algorithms.
Each line in the plots corresponds to experimental data for one particular backdoor.
The left plot in Fig.~\ref{fig:check-easy-g3} shows the SAT solving time for problems $C[\beta/B] \in S(P)$ for different values of $K$, the middle plot shows the time used for proof checking with DRAT-trim, and the right plot depicts the proof sizes.
We see that the minimal values of all these parameters are reached for $K$ around 10--20 for the studied SAT instances; moreover, the plots clearly show that solving all these subproofs individually  is infeasible.
Therefore, in our experiments we used $K=20$ for all formulas.

\begin{figure}[h]
    \centering
    \includegraphics[width=\textwidth]{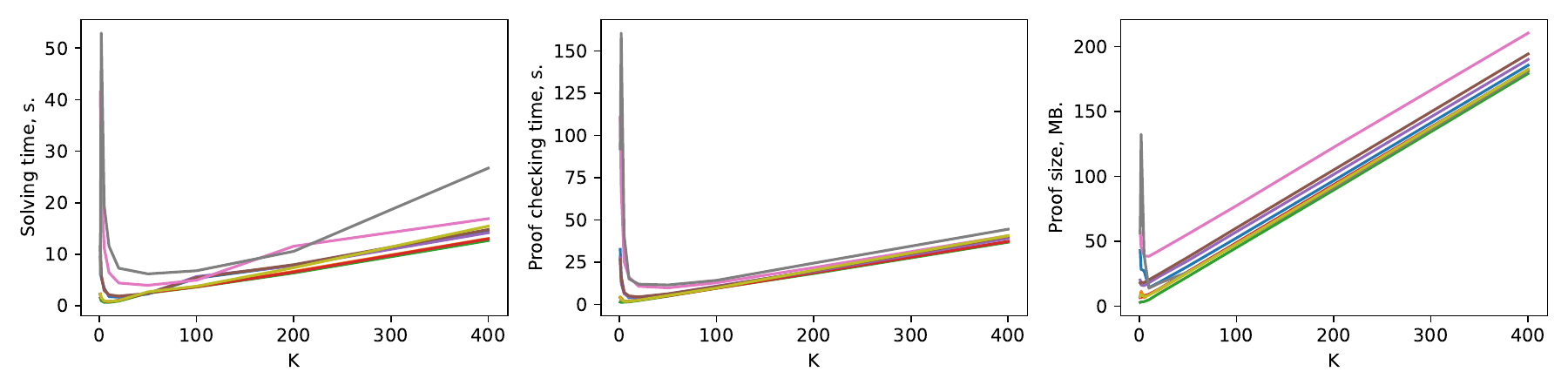}
    \caption{Proving unsatisfiability and checking proofs for simple subproblems}
    \label{fig:check-easy-g3}
\end{figure}

\section{Experiments}
\label{sec:experiments}

We performed computational experiments in order answer the following research questions.

\emph{RQ1}. Which measure of solver workload is best for d-hardness estimation: time in seconds, number of unit propagations, or number of conflicts?

\emph{RQ2}. Can the described techniques improve generation of sets $B$ that would allow solving SAT instances faster than by just using a SAT solver? 

\emph{RQ3}. Can the proposed decomposition approach decrease the time needed to check unsatisfiability proofs?

\subsection{Benchmarks}

We conducted our computational experiments using test instance classes which were used in the paper~\cite{CP2021}. 
One of them is the well-known class of unsatisfiable formulas constructed using the \emph{sgen} generator~\cite{Spence15}. 
The second one is a class formed by SAT encodings of some examples of the Logical Equivalence Checking~(LEC) problem. 
In more detail, the paper~\cite{CP2021} considered problems in which one has to prove the equivalence of pairs of Boolean circuits. 
The considered circuits represent different algorithms performing sorting of arbitrary $k$ natural numbers of length $l$-bit in binary form. 
Three sorting algorithms were considered: Bubble sort, Selection sort~\cite{Cormen90}, and Pancake sort~\cite{Gates79}. 
Corresponding instances are denoted as $BvP_{k,l}$ for ``Bubble vs. Pancake'' test, $BvS_{k,l}$ for ``Bubble vs. Selection'', and $PvS_{k,l}$ for ``Pancake vs. Selection''. 
For example, the CNF formula with name $PvS_{7,4}$ encodes the LEC problem for two circuits $S_1$, $S_2$, which perform sorting of seven arbitrary 4-bit numbers, where $S_1$ implements Pancake sort and $S_2$ implements Selection sort.   

\subsection{Solver workload measures for d-hardness estimation}
The first set of experiments is dedicated to answering RQ1.
We fixed the SAT instance, evolutionary algorithm, time limit, and other parameters, and only varied the solver workload measure used for d-hardness estimation: namely, the number of unit propagations, the number of conflicts, or the execution time (in seconds) of the SAT solver.
We then ran the experiments, making 20 independent runs with each algorithm setup, and compared the results (sets $B$) according to one simple parameter:
the time needed to solve the formula using the found backdoor $B$.

The experiments were set up as follows.
We considered the SAT instance $PvS_{7,4}$ (see its characteristics in Section~\ref{sec:main-experiments}).
For each SAT solver (\emph{CaDiCaL}, \emph{Glucose~3}) and each measure of solver workload (propagations, conflicts, time) we did 20 independent runs of the genetic algorithm (with parameters $G = 8, E = 2$) using an initial random sample size $N = 1000$.
Each run was limited to one hour using 16 parallel threads on a computer with a 32-core Intel(R) Xeon(R) Gold 6242 CPU @ 2.80~GHz.

For each run, we selected the backdoor set $B$ with the smallest found d-hardness estimate, and solved the formula $C$ using this set $B$.
As a result, for each set of experiments we have a sample of 20 execution times corresponding to 20 different sets $B$.

Violin plots of these samples are shown in Fig.~\ref{fig:measures-violins}: for \emph{CaDiCaL} on the left, and for \emph{Glucose~3} on the right.
\begin{figure*}
    \centering
    \includegraphics[width=\textwidth]{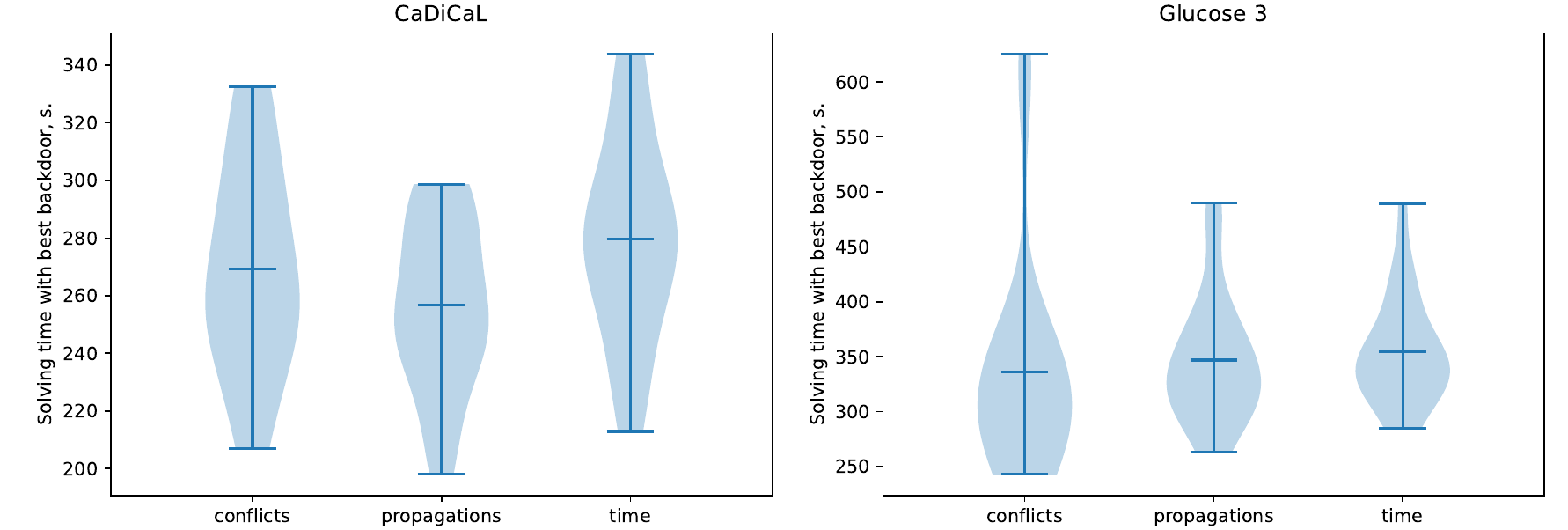}
    \caption{SAT solving times for instance $PvS_{7,4}$ with solvers CaDiCaL (left) and Glucose~3 (right) using decomposition sets $B$ constructed using different solver hardness measures (conflicts, propagations, time)}
    \label{fig:measures-violins}
\end{figure*}
From the data in Fig.~\ref{fig:measures-violins} we see that there is no considerable difference in the results for different measures.
To verify this conclusion, we used the Mann-Whitney U-test~\cite{MannWhitney}, which showed that for both solvers (1) results with propagations and conflicts are statistically indistinguishable, and (2) results with time are statistically worse than with propagations (p-value equals $0.043$ for significance level $0.05$).
We thus conclude for RQ1 that one may use either propagations or conflicts to get good resulting d-hardness estimates.
In order to be consistent with the results from~\cite{CP2021}, in all further experiments we used the number of unit propagations as the solver workload measure for d-hardness estimation.

\subsection{Incremental preprocessing: using simple subproblems for faster d-hardness estimation}

In this section we describe the results of experiments in which we took into account ``simple'' subproblems $C[\beta /B]\in S(P)$, where $P$ is some polynomial-time sub-solver (see Section~\ref{sec:additional}). 
Also in these experiments we used different variants of the reduced search space $B_0 \subset X$ w.r.t. techniques described in Section~\ref{sec:search-space-reduction}.

To check the impact of the aforementioned techniques on the efficiency of the search procedure, we ran experiments with d-hardness for $PvS_{7,4}$ using unit propagations as the hardness measure, running the genetic algorithm ($G = 8, E = 2$) with a time limit of one hour, comparing runs of the algorithm with and without these modifications.

Plots in Fig.~\ref{fig:incremental-gad-full} on left demonstrate how the algorithms with and without incremental preprocessing behave when run starting from the entire set $X$.
It is evident that this modification helps the original algorithm from~\cite{CP2021} descend to solutions with smaller d-hardness estimations faster, and explore more backdoors (see plot on the right).
However, the effect is not sufficient: in one hour, even the algorithm that uses incremental preprocessing finds sets $B$ with more than 2000 variables, which cannot be used in practice to solve SAT for the considered formula.

Plots in Fig.~\ref{fig:incremental-gad} show similar results, but when both algorithms are initialized with the initial heuristically constructed set $B_0$, $|B_0| = 200$, and run with a time limit of 30 minutes.
The figure on the left demonstrates that the algorithm that uses incremental preprocessing begins encountering ``hard'' subproblems $C[\beta/B] \notin S(P)$ only after reaching
relatively small sets $B$ with $|B|$ of about 20...25, and for all sets $B$ encountered before it turned to be $C[\beta /B]\in S(P)$ for all $\beta$'s in the corresponding random sample (here, the Unit Propagation rule is used as $P$).

\begin{figure}[t]
    \centering
    \includegraphics[width=\textwidth]{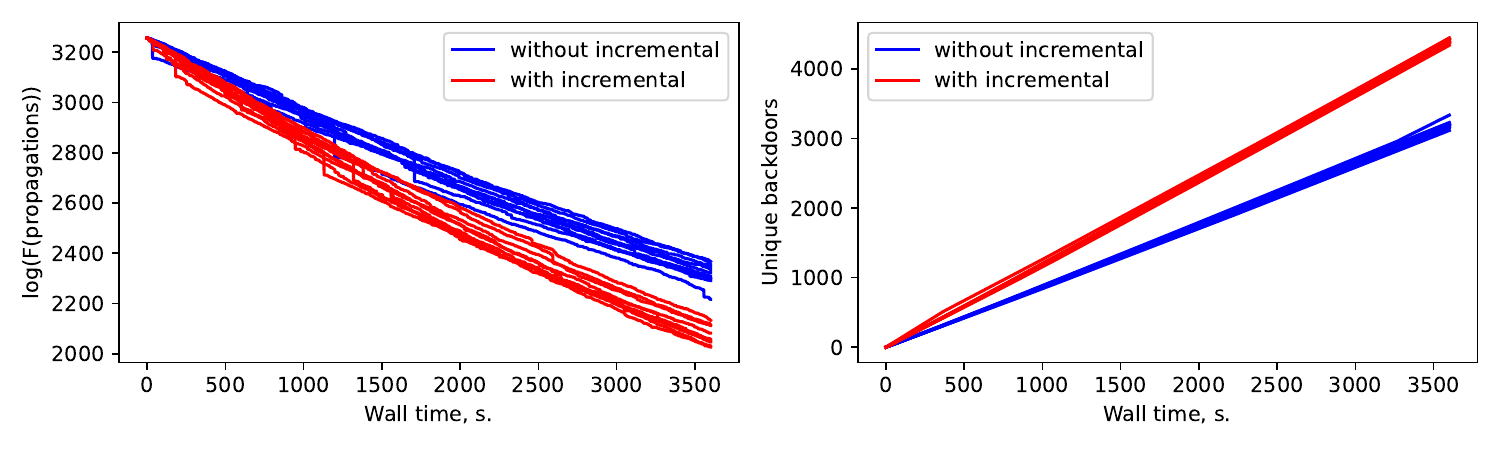}
    \caption{Effect of using UP (via incremental SAT solving) as a preprocessor in d-hardness function estimation when starting the search from the entire set of variables $X$}
    \label{fig:incremental-gad-full}
\end{figure}

\begin{figure}[t]
    \centering
    \includegraphics[width=\textwidth]{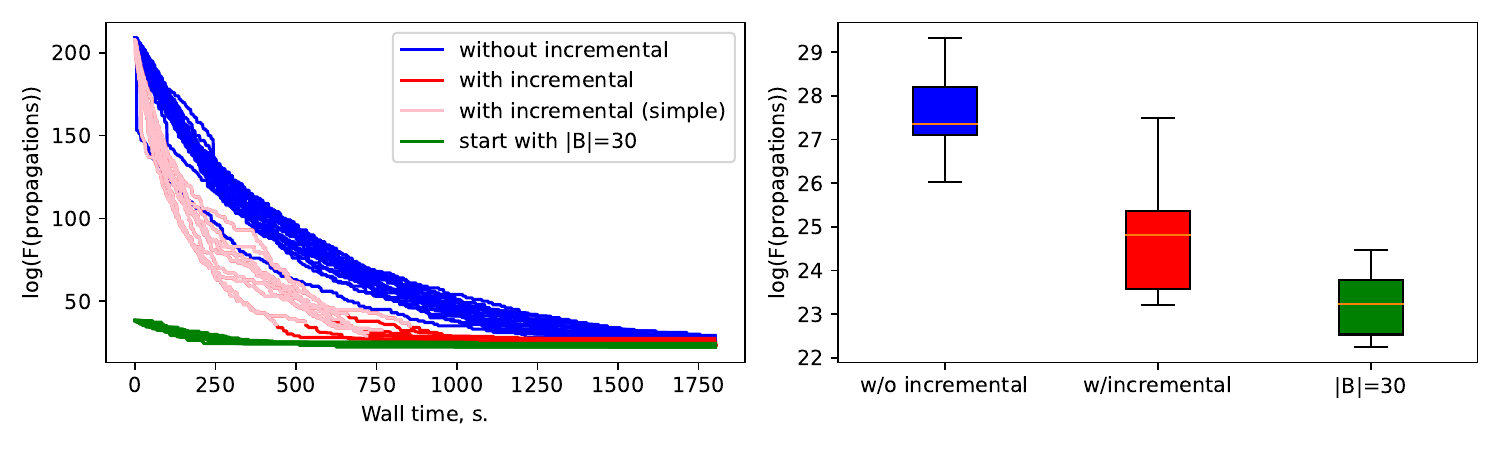}
    \caption{Effect of using UP (via incremental SAT solving) as a preprocessor in d-hardness function estimation when starting the search from a heuristically constructed set $B_0$, $|B_0| = 200$}
    \label{fig:incremental-gad}
\end{figure}

The observation that hard subproblems only start to appear for smaller backdoors  also leads to the idea of using a sufficiently large set $B_0$ (say, $|B_0| = 200$ variables) as described in Section~\ref{sec:search-space-reduction}, but initializing the algorithm with much smaller initial solutions (say, $|B| = 30$).
This way, the search space is still limited to $B_0$, but the solutions are initialized in such a way that the algorithm soon starts encountering hard subproblems.
The green plots in Fig.~\ref{fig:incremental-gad} show results for the algorithm that is initialized in this way.
Note also the plot in Fig.~\ref{fig:incremental-gad} on the right: it shows the final distributions of d-hardness estimations found by the three compared algorithms.
It is clearly evident that the best results are achieved when using incremental preprocessing and starting the search from a set $B$, $|B| = 30$.

As additional motivation for using small initial sets $B$, consider the plots in Fig.~\ref{fig:subps}:
the left plot shows the value of the $w^i$ heuristic weight introduced in Section~\ref{sec:search-space-reduction} for the 600 variables with the highest values of this heuristic for the formula $PvS_{7, 4}$; the right plot shows the value of parameter $\rho$ estimated (with sample size $N = 10000$) for a backdoor $B$ that is built from the best $1, 2, 10, 100, \ldots, 600$ variables (according to $w_i$).
From the right plot we see that the value of $\rho$ becomes equal to 1.0 already for $|B| = 20$.
This means that we can safely start the d-hardness estimation algorithms from sets $B$ of about this size, and the corresponding subproblems $C[\beta/B]$ will be solved efficiently.
The estimation of the reasonable cardinality of initial solutions can be done simultaneously with the search space reduction procedure.

\begin{figure}
    \centering
    \includegraphics[width=\textwidth]{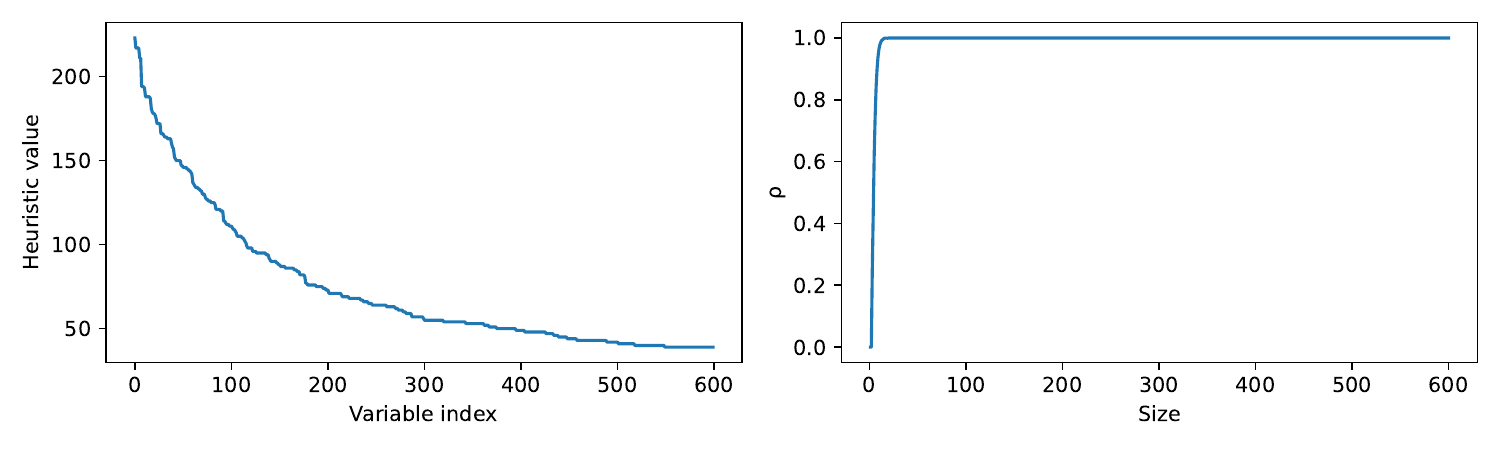}
    \caption{Selecting initial set size for $PvS_{4,7}$}
    \label{fig:subps}
\end{figure}

\subsection{Main experiments}
\label{sec:main-experiments}

In the main computational experiments, we launched the modified d-hardness algorithm with incremental preprocessing augmented with the space reduction heuristic ($|B_0| = 200$) and initialization with sets $B$ of size 30.
Here we do not consider the original d-hardness approach~\cite{CP2021}, since we already showed above that it cannot compete with the optimized version.

Each run of the algorithm on each SAT instance with each SAT solver was repeated 10 times with a time limit of one hour using 16 threads on a computer with a 32-core Intel(R) Xeon(R) Gold 6242 CPU @ 2.80~GHz.
The d-hardness estimation algorithm used a sample size of $N = 1000$.

After each run, we determined the best resulting set $B$ (according to the found d-hardness estimation value), and solved the corresponding SAT formula using this set.
The solving experiments were run on a computer with a 32-Core Intel(R) Xeon Gold 6338 CPU @ 2.00~GHz.
Each such experiment was repeated five times, and we took the average execution time as the result to mitigate possible variation in running time due to operating system effects (the standard deviation of execution times for one backdoor did not exceed $5\%$).
In addition, we stored, solved, and checked the unsatisfiability proofs with DRAT-trim.

\begin{table}[t]
    \centering
    \begin{tabular}{cccccccc}
    \toprule
    \multirow{3}{*}{Instance} & \multirow{3}{*}{$|X|$} & \multicolumn{3}{c}{CaDiCaL (cd)} & \multicolumn{3}{c}{Glucose 3 (g3)}\tabularnewline
     \cmidrule(lr){3-5} \cmidrule(lr){6-8}
    & & \multirow{2}{*}{Solve, s} & \multicolumn{2}{c}{Proof} & \multirow{2}{*}{Solve, s} & \multicolumn{2}{c}{Proof}\tabularnewline 
    \cmidrule(lr){4-5} \cmidrule(lr){7-8}
    & & & Size, MB & Check, s & & Size, MB & Check, s\tabularnewline
    \midrule
    $PvS_{7,4}$ & 3244 & 460 & 679 & 1601 & 986 & 1803 & 3206\tabularnewline
    $BvP_{7,6}$ & 3492 & 487 & 561 & 1199 & 1480 & 2481 & 6644\tabularnewline
    $BvS_{7,7}$ & 4462 & 837 & 814 & 1617 & 3559 & 4596 & 15084\tabularnewline
    $BvP_{8,4}$ & 3013 & 1488 & 1510 & 5647 & 3598 & 4704 & 21503\tabularnewline
    $sgen$ & 150 & 2022 & 2332 & 5898 & 8404 & 12823 & 37724\tabularnewline
    \bottomrule
    \end{tabular}
    \caption{Data on the considered CNF formulas and SAT solvers CaDiCaL~(cd) and Glucose~3~(g3): solving times, unsatisfiability proofs sizes, and proof checking times}
    \label{tab:original}
\end{table}

In this main experimental section, we used five SAT instances which were previously used for benchmarking the approach proposed in~\cite{CP2021}.
Table~\ref{tab:original} contains data on solving these formulas with SAT solvers CaDiCaL and Glucose~3, and also checking the generated unsatisfiability proofs with DRAT-trim.

The experimental results are summarized in Table~\ref{tab:main}.
For each instance $C$ and SAT solver $A$ the table shows:
\begin{itemize}
    \item the average size of the found set $B$;
    \item for the case of using a single backdoor to solve the SAT formula: 
    \begin{itemize}
        \item the average and relative standard deviation of the ratio $r_{B,A}$, where each ratio is calculated as the time used to solve the formula using a decomposition generated by set $B$ divided by the time used to solve it with solver $A$ without decomposition;
        \item a similar ratio $\pi_{B,A}$, but calculated based on proof checking time;
    \end{itemize}
    \item for the case of using a pair of backdoors to solve the SAT formula: similar measures $r_{\Gamma,A}$ and $\pi_{\Gamma,A}$, where $\Gamma$ is the set of ``hard'' subproblems from the corresponding Cartesian product.
\end{itemize}

\begin{table}[htbp]
    \centering
    \begin{tabular}{cccccccc}
    \toprule
    \multirow{2}{*}{$C$} & \multirow{2}{*}{$A$} & \multirow{2}{*}{Avg. $|B|$} & \multicolumn{2}{c}{Single backdoors} & \multicolumn{2}{c}{Pairs of backdoors}\tabularnewline
    \cmidrule(lr){4-5} \cmidrule(lr){6-7}
    & & & $r_{B,A}$ & $\pi_{B,A}$ & $r_{\Gamma,A}$ & $\pi_{\Gamma,A}$\tabularnewline
    \midrule
    \multirow{2}{*}{$PvS_{7, 4}$} & cd & $14.5$ & $0.41 \pm 0.11$ & $0.30 \pm 0.16$ & $0.39 \pm 0.11$ & $0.22 \pm 0.12$\tabularnewline
    & g3 & $14.3$ & $0.33 \pm 0.14$ & $0.32 \pm 0.21$ & $0.28 \pm 0.12$ & $0.22 \pm 0.10$\tabularnewline
    \cmidrule(lr){1-3} \cmidrule(lr){4-5} \cmidrule(lr){6-7}
    \multirow{2}{*}{$BvP_{7, 6}$} & cd & $13.3$ & $0.79 \pm 0.12$ & $0.57 \pm 0.14$ & $0.74 \pm 0.11$ & $0.47 \pm 0.13$\tabularnewline
     & g3 & $14.1$ & $0.48 \pm 0.17$ & $0.32 \pm 0.26$ & $0.41 \pm 0.14$ & $0.23 \pm 0.21$\tabularnewline
    \cmidrule(lr){1-3} \cmidrule(lr){4-5} \cmidrule(lr){6-7}
    \multirow{2}{*}{$BvS_{7, 7}$} & cd &$13.8$ & $0.80 \pm 0.10$ & $0.88 \pm 0.14$ & $0.70 \pm 0.12$ & $0.66 \pm 0.11$\tabularnewline
     & g3 & $14.4$ & $0.23 \pm 0.08$ & $0.13 \pm 0.13$ & $0.21 \pm 0.11$ & $0.10 \pm 0.14$\tabularnewline
     \cmidrule(lr){1-3} \cmidrule(lr){4-5} \cmidrule(lr){6-7}
     \multirow{2}{*}{$BvP_{8, 4}$} & cd & $13.9$ & $1.09 \pm 0.06$ & $0.80 \pm 0.10$ & $0.93 \pm 0.06$ & $0.51 \pm 0.11$\tabularnewline
      & g3 & $14.4$ & $0.72 \pm 0.06$ & $0.49 \pm 0.12$ & $0.60 \pm 0.08$ & $0.31 \pm 0.12$\tabularnewline
      \cmidrule(lr){1-3} \cmidrule(lr){4-5} \cmidrule(lr){6-7}
      \multirow{2}{*}{$sgen$} & cd & $19.0$ & $0.20 \pm 0.08$ & $0.12 \pm 0.15$ & -- & -- \tabularnewline
      & g3 & $19.6$ & $0.05 \pm 0.14$ & $0.01 \pm 0.29$ & -- & -- \tabularnewline
      \bottomrule
    \end{tabular}
    \caption{Results with solving and checking unsatisfiability proofs via decompositions using one or two backdoors}
    \label{tab:main}
\end{table} 

Looking at the results in Table~\ref{tab:main}, we note that for the vast majority of cases,
the proposed algorithm found sets $B$ that allow solving the original CNF formula faster than just by applying a SAT solver. 
Indeed, for most cases, $r_{B,A} < 1$.
Also, the corresponding ratio for proof checking $\pi_{B,A}$ is smaller than $1$ in all cases.
Moreover, solving SAT and checking proofs with the use of a pair of backdoors is, on average, always faster 
than with one backdoor: indeed, for all formulas $C$ we have $r_{\Gamma,A} < r_{B,A}$ and $\pi_{\Gamma,A} < \pi_{B,A}$.
What is also interesting, the reduction in proof checking time is typically even more significant than the reduction of the solving time, i.e. in most cases we have $\pi_{B,A} < r_{B,A}$ and $\pi_{\Gamma,A} < r_{\Gamma,A}$.
This is a promising result, since the proof checking times for the original monolithic proofs are always (in our experiments) larger than the SAT solving times (see Table~\ref{tab:original}).
So, we can give a positive answer to both RQ2 and RQ3: the proposed approach allows decreasing both the time needed to solve SAT formulas and the time needed to check the corresponding unsatisfiability proofs.

Note that in all experiments above, the efficiency of SAT solving with found sets $B$ was considered w.r.t. sequential SAT solving, i.e. all problems were solved using one SAT solver on one CPU.
However, we emphasize that the nature of SAT decompositions implies the possibility to solve the generated subproblems $C[\beta/B]$ (or $C[\gamma]$) independently and in parallel.
In this work, we did not run such experiments, but instead simulated the potential parallel speedup using a simple queue-based algorithm: all subproblems are put in a queue, and ``solved'' by a pool of workers.
As a result, this algorithm outputs the wall-clock time when the last subproblem is solved.
Thus, this time may be used as an estimate of how the simplest parallel SAT solving strategy (i.e. without clause sharing and other well-known parallel SAT solving techniques) that uses one or several sets $B$ might work like.

The corresponding speedups for several considered formulas w.r.t. the solver Glucose~3 are depicted in Fig.~\ref{fig:speedups} for the cases of using $1..32$ parallel workers.
Here, in addition to results in Table~\ref{tab:main} for single backdoors and pairs of backdoors, we additionally considered triples of backdoors.
Note that, for all cases, the more backdoors we use, the larger the speedup becomes.
This is explained by the result that when we, e.g. use a pair of backdoors instead of a single backdoor, we, on the one hand, get more hard problems to solve, but on the other hand, these problems are typically simpler, and are  solved faster.
Analysis on these experiments shows that the best results can be achieved when using triples of backdoors.

\begin{figure}[htbp]
    \centering
    \includegraphics[width=0.6\textwidth]{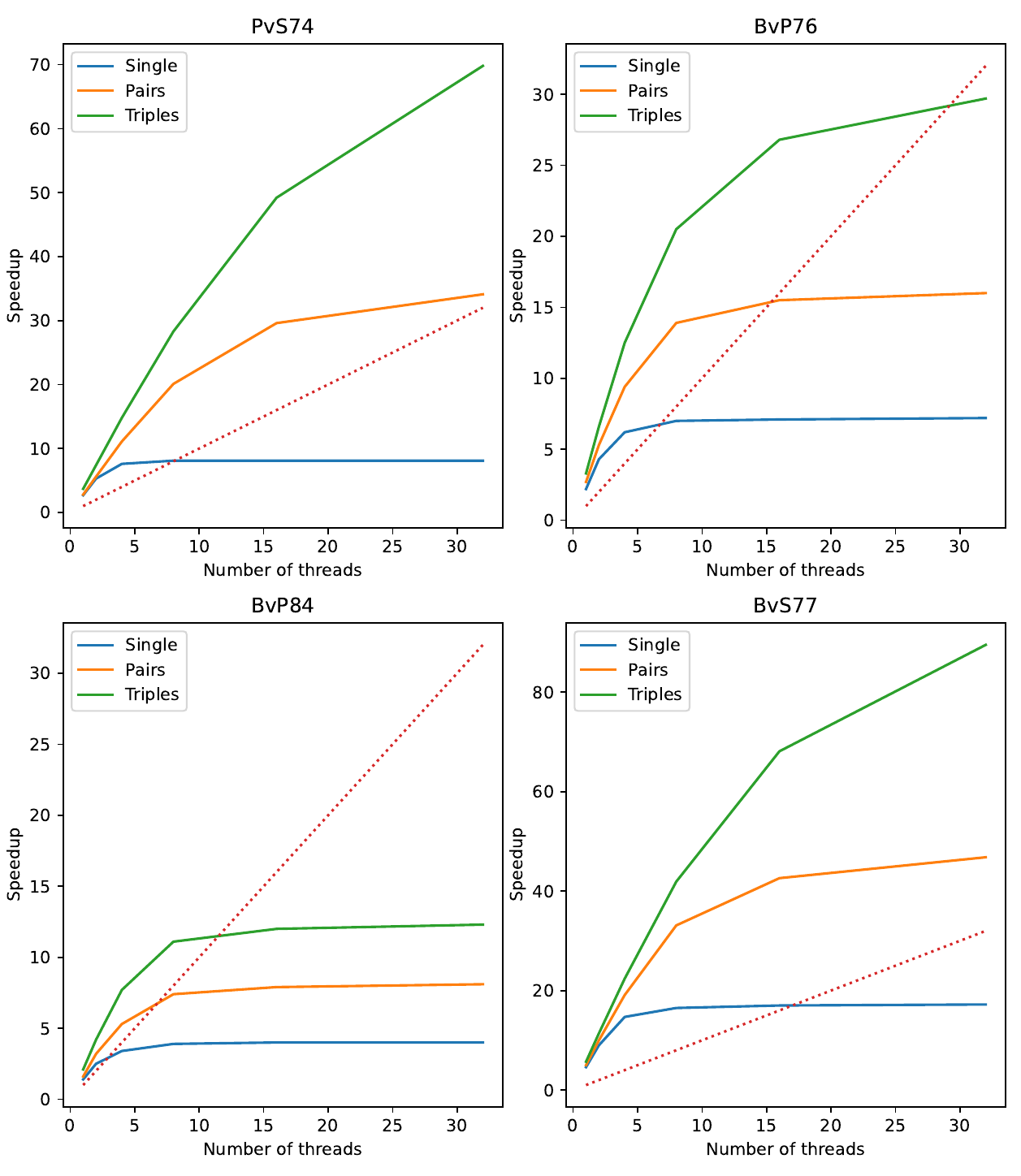}
    \caption{Speedups for the case of simulating parallel SAT solving with a set of independent Glucose~3 solvers}
    \label{fig:speedups}
\end{figure}

\section{Related Work}
\label{sec:related-work}

It is hardly possible to accurately indicate the paper in which the decompositions of SAT were first built, since the very idea to decompose some hard combinatorial problem into a family of problems of lower dimensionality is very natural. 
Apparently, in~\cite{Hyvar06} the SAT Partitioning conecept was introduced, which can be considered as a special type of decomposition of a Boolean formula. 
The Partitioning approach was further developed in~\cite{Hyvar10,Hyvar11,HyvarPhD} and a number of subsequent articles. 
A good overview of the various methods for solving SAT in parallel is given in~\cite{HyvarPhD}. 
Later, the idea of using lookahead strategies to build SAT Partitionings with their subsequent processing by CDCL SAT solvers was developed within the framework of an approach called Cube-and-Conquer~\cite{CC2012}. 
It was the use of Cube-and-Conquer that made it possible to solve a number of hard combinatorial problems and even prove/refute some mathematical hypotheses using massive parallel computations~\cite{KonLisitsa14,Heule16,Heule18}.

The ideas of decomposing hard SAT instances have proven to be extremely fruitful in application to cryptanalysis. 
G.~Bard in~\cite{Bard09} notes that SAT solvers are a very efficient tool for reducing the number of solutions that need to be enumerated and, accordingly, can be used to build various cryptographic attacks. 
One of the most numerous classes of attacks of this type are guess-and-determine attacks. 
The basic idea of such an attack is to decompose the formula which encodes the considered cryptographic problem into a family of significantly simpler formulas with the expectation that the total number of calculations will be significantly lower than for a brute force attack. 
Over the past 20 years, a large number of papers have been published in which SAT solvers have been applied to cryptanalysis problems. 
We mention here only the most significant articles on this topic.
One of the first works in which a cryptanalysis problem was reduced to SAT is the article~\cite{Mass2000}. 
A guess-and-determine attack using SAT solvers for the round-reduced DES cipher was constructed in~\cite{CourtBard07}. 
In~\cite{SZBP11}, a completely practical SAT-based guess-and-determine attack was presented for the well-known stream cipher A5/1, which has been used in the GSM standard for cell traffic encryption for many years. 
This attack was implemented in the voluntary computing project SAT@home~\cite{SAT-at-home-2012}. 
The paper~\cite{Soos09} describes the \emph{cryptominisat} SAT solver, a number of technical details of which, according to its author, were designed specifically for cryptanalysis problems. 
In the articles~\cite{Eibach08,Soos09,SZ15,SZ16} and a number of others, SAT-based attacks (including guess-and-determine ones) were built for a number of stream ciphers, among which was the Bivium cipher, which is a weakened version of the well-known lightweight stream cipher Trivium~\cite{Cann2006} (which became one of the winners of the eSTREAM contest in 2008). 
The article~\cite{SZOKI18} described a method for automatically constructing SAT-based guess-and-determine attacks applicable to a wide class of cryptographic functions. 
Also in this work, guess-and-determine attacks were built for a number of ciphers, which at that time were the best known ones. 
Attacks based on the use of SAT solvers for cryptographic hash functions of the MD family have been described in~\cite{Mironov2006,De2007,Gribanova2018,Gribanova20,Zaikin2022}.

It is easy to understand that we can always build a trivial decomposition of formula $C$ over $X$, breaking it down into $2^{|X|}$ formulas by substituting all possible values of variables from $X$ to $C$. 
SAT for each such formula is solved in time that is linear in $|C|$. 
As we said above, there are many examples of decompositions of the original problem into $2^{|B|}$ subproblems, where $B \subset X$, $|B| \ll |X|$, in which any problem from the decomposition is solved by some polynomial algorithm $P$.
In this case, $B$ is a Strong Backdoor Set (SBS) for $C$ w.r.t. algorithm $P$. 
This concept was introduced in~\cite{Williams03}. 
In subsequent years, the backdoors topic was actively developed in many directions, including parameterized complexity~\cite{Fichte2015,Gaspers2012a,FPTBook12,Gaspers2012c,Misra2013,Hemaspaandra2021}.

It was emphasized in~\cite{Ansoteg08} that knowing an SBS for some Boolean formula automatically gives some upper estimate of the hardness of that formula. 
Also~\cite{Ansoteg08} explored several tree-like measures of hardness in the context of SAT and demonstrated their relationship with backdoor hardness.

In the article~\cite{CP2021}, based on the analysis of the ideas from~\cite{Ansoteg08}, we introduced the concept of decomposition hardness or d-hardness of a SAT instance (w.r.t. some deterministic and complete SAT solving algorithm). 
The key idea behind d-hardness is that we can associate with an arbitrary decomposition set $B \subseteq X$ a special random variable $\xi_B$ with finite expectation and variance, and get an upper bound for the hardness of $C$ in the form $2^{|B|} \cdot \ev{\xi_B}$, where $\ev{\xi_B}$ is the expected value of $\xi_B$.
This approach makes it possible to evaluate the hardness of SAT instances through evaluation of $\ev{\xi_B}$, and the Monte Carlo method can be used for this purpose. 
As noted in~\cite{CP2021}, the estimates of d-hardness obtained in this way can, in theory, be made arbitrarily accurate by increasing the volume $N$ of the random sample.
However, in practice, an issue with the accuracy of estimates can arise due to the large variance of the random variable $\xi_B$, which, in turn, is a consequence of the heavy-tailed behavior phenomenon of modern complete SAT solvers~\cite{GomesSabh2021}.

The probabilistic generalization of the SBS notion introduced in~\cite{AAAI2022} has a number of attractive properties. 
The basic idea of such a generalization is that some polynomial algorithm $P$ is allowed to solve not all problems in the family obtained as a result of decomposition over the set $B$, but only some fraction $\rho$ of such problems (in this case, $B$ is called a $\rho$-backdoor). 
Of practical value are $\rho$-backdoors of low cardinality with the value of $\rho$ close to 1. 
In the case when $\rho = 1$, $B$ is an SBS. 
As shown in~\cite{AAAI2022}, we can efficiently estimate the value of $\rho$ for a given set $B$ using a simple Monte Carlo test and still guarantee the accuracy of the resulting estimates due to the fact that the random variable $\xi_B$ associated with a $\rho$-backdoor $B$ is a Bernoulli random variable and, thus, it has variance $Var(\xi_B) \leq 1/4$.

The problem of finding a decomposition set which provides a minimal d-hardness value is reduced to minimization of a pseudo-Boolean fitness function, and metaheuristic optimization algorithms are used to minimize it. 
In this article, we used for this a modified version of the program tool EvoGuess~\cite{evostar}, which was previously used to build SAT-based guess-and-determine attacks for a number of keystream generators.

This article significantly expands the results of~\cite{CP2021}: we proposed several techniques that significantly reduce the search efficiency; in the new experiments, we used search space reduction heuristics; we also demonstrated the possibility of using constructed decomposition sets to reduce the time needed for checking the proofs for hard unsatisfiable formulas; finally, we performed a more accurate experimental analysis, in which all experiments were repeated independently at least ten times.

\section{Conclusion}
\label{sec:conclusion}

In this paper, we presented an approach to estimating the hardness of Boolean formulas in the context of the Boolean satisfiabililty problem (SAT). 
The proposed hardness measure, d-hardness, is based on the fact that it is possible to estimate the running time of a complete SAT solving algorithm $A$ w.r.t. some subset $B$ of variables of the original formula. 
For this purpose, first, the decomposition of the formula is constructed using all possible assignments of variables from $B$. 
If we use a deterministic complete SAT solving algorithm $A$, then we can associate with such a decomposition a special random variable with finite expected value and variance. 
It is shown in the paper that the running time of $A$ on the decomposition of $C$ induced by $B$ can be expressed via the expected value (expectation) of the introduced random variable. 
For a specific set $B$, we can construct the estimation of this expectation using the Monte Carlo method. 
Finding the set $B$ for which the corresponding estimation is minimal is of particular interest. 
This problem can be formulated in the framework of pseudo-Boolean optimization: for this purpose, we construct a special fitness function and consider the problem of its minimization over some finite search space which is essentially a Boolean hypercube.

To conlude, we have proposed several important modifications of the algorithm for d-hardness estimation previously presented in~\cite{CP2021}.
In particular, we drew a connection of d-hardness with the $\rho$-backdoors concept~\cite{AAAI2022}, and proposed an optimization of the d-hardness estimation algorithm that is based on the use of the concept of $\rho$-backdoors.
We also showed how to use an arbitrary set of variables $B$ associated with the found d-hardness estimation to build a decomposed unsatisfiability proof, and demonstrated that these proofs can be checked much faster than the monolithic proof for the original SAT formula.


\subsubsection*{Acknowledgment}

This work was supported by the Analytical Center for the Government of the Russian Federation (IGK~000000D730321P5Q0002), agreement No.~70-2021-00141.

\bibliographystyle{plain}
\bibliography{paper}

\begin{thebibliography}{10}

\bibitem{Alekhn02}
Michael Alekhnovich, Jan Johannsen, Toniann Pitassi, and Alasdair Urquhart.
\newblock An exponential separation between regular and general resolution.
\newblock In {\em {STOC}}, pages 448--456. {ACM}, 2002.

\bibitem{Ansoteg08}
Carlos Ans{\'{o}}tegui, Maria~Luisa Bonet, Jordi Levy, and Felip Many{\`{a}}.
\newblock Measuring the hardness of {SAT} instances.
\newblock In {\em {AAAI}}, pages 222--228, 2008.

\bibitem{Balint2012}
Adrian Balint and Uwe Sch{\"{o}}ning.
\newblock Choosing probability distributions for stochastic local search and the role of make versus break.
\newblock In {\em {SAT}}, volume 7317 of {\em LNCS}, pages 16--29, 2012.

\bibitem{Bard09}
Gregory~V. Bard.
\newblock {\em Algebraic {{Cryptanalysis}}}.
\newblock {Springer US}, 2009.

\bibitem{benSasson2004}
Eli Ben{-}Sasson, Russell Impagliazzo, and Avi Wigderson.
\newblock Near optimal separation of tree-like and general resolution.
\newblock {\em Comb.}, 24(4):585--603, 2004.

\bibitem{benSasson2001}
Eli Ben{-}Sasson and Avi Wigderson.
\newblock Short proofs are narrow - resolution made simple.
\newblock {\em J. {ACM}}, 48(2):149--169, 2001.

\bibitem{BiereDAC99}
Armin Biere, Alessandro Cimatti, Edmund~M. Clarke, Masahiro Fujita, and Yunshan Zhu.
\newblock Symbolic model checking using {SAT} procedures instead of {BDDs}.
\newblock In {\em {DAC}}, pages 317--320. {ACM} Press, 1999.

\bibitem{BiereTACAS99}
Armin Biere, Alessandro Cimatti, Edmund~M. Clarke, and Yunshan Zhu.
\newblock Symbolic model checking without {BDDs}.
\newblock In {\em {TACAS}}, volume 1579 of {\em LNCS}, pages 193--207, 1999.

\bibitem{bonet1998}
Maria~Luisa Bonet, Juan~Luis Esteban, Nicola Galesi, and Jan Johannsen.
\newblock Exponential separations between restricted resolution and cutting planes proof systems.
\newblock In {\em {FOCS}}, pages 638--647, 1998.

\bibitem{Boros02}
Endre Boros and Peter~L. Hammer.
\newblock Pseudo-boolean optimization.
\newblock {\em Discret. Appl. Math.}, 123(1-3):155--225, 2002.

\bibitem{ABC-Misch10}
Robert~K. Brayton and Alan Mishchenko.
\newblock {ABC:} an academic industrial-strength verification tool.
\newblock In {\em {CAV}}, volume 6174 of {\em LNCS}, pages 24--40, 2010.

\bibitem{Buss1998}
Samuel~R. Buss and Gy{\"{o}}rgy Tur{\'{a}}n.
\newblock Resolution proofs of generalized pigeonhole principles.
\newblock {\em Theor. Comput. Sci.}, 62(3):311--317, 1988.

\bibitem{Cai2021}
Shaowei Cai and Xindi Zhang.
\newblock Deep cooperation of {CDCL} and local search for {SAT}.
\newblock In {\em {SAT}}, volume 12831 of {\em LNCS}, pages 64--81, 2021.

\bibitem{Cann2006}
Christophe~De Canni{\`{e}}re.
\newblock Trivium: {A} stream cipher construction inspired by block cipher design principles.
\newblock In {\em {ISC}}, volume 4176 of {\em LNCS}, pages 171--186, 2006.

\bibitem{ChangLee73}
Chin-Liang Chang and Richard Char-Tung Lee.
\newblock {\em Symbolic Logic and Mechanical Theorem Proving}.
\newblock Academic Press, Inc., 1973.

\bibitem{Chern52}
Herman Chernoff.
\newblock A measure of asymptotic efficiency for tests of a hypothesis based on the sum of observations.
\newblock {\em Ann. Math. Stat.}, 23(4):493--507, 1952.

\bibitem{Cook71}
Stephen~A. Cook.
\newblock The complexity of theorem-proving procedures.
\newblock In {\em {STOC}}, pages 151--158. {ACM}, 1971.

\bibitem{CookRekh79}
Stephen~A. Cook and Robert~A. Reckhow.
\newblock The relative efficiency of propositional proof systems.
\newblock {\em J. Symb. Log.}, 44(1):36--50, 1979.

\bibitem{Cormen90}
Thomas Cormen, Charles Leiserson, and Ronald Rivest.
\newblock {\em Introduction to Algorithms}.
\newblock MIT Press, 1990.

\bibitem{CourtBard07}
Nicolas~T. Courtois and Gregory~V. Bard.
\newblock Algebraic cryptanalysis of the data encryption standard.
\newblock In {\em {IMACC}}, volume 4887 of {\em LNCS}, pages 152--169, 2007.

\bibitem{DPLL1962}
Martin Davis, George Logemann, and Donald Loveland.
\newblock A machine program for theorem-proving.
\newblock {\em Commun. ACM}, 5(7):394–397, 1962.

\bibitem{DP1960}
Martin Davis and Hilary Putnam.
\newblock A computing procedure for quantification theory.
\newblock {\em J. ACM}, 7(3):201–215, 1960.

\bibitem{De2007}
Debapratim De, Abishek Kumarasubramanian, and Ramarathnam Venkatesan.
\newblock Inversion attacks on secure hash functions using satsolvers.
\newblock In {\em {SAT}}, volume 4501 of {\em LNCS}, pages 377--382, 2007.

\bibitem{Doerr17}
Benjamin Doerr, Huu~Phuoc Le, R{\'{e}}gis Makhmara, and Ta~Duy Nguyen.
\newblock Fast genetic algorithms.
\newblock In {\em {GECCO}}, pages 777--784. {ACM}, 2017.

\bibitem{Dowling84}
William~F. Dowling and Jean~H. Gallier.
\newblock Linear-time algorithms for testing the satisfiability of propositional horn formulae.
\newblock {\em J. Log. Program.}, 1(3):267--284, 1984.

\bibitem{Eibach08}
Tobias Eibach, Enrico Pilz, and Gunnar V{\"{o}}lkel.
\newblock Attacking bivium using {SAT} solvers.
\newblock In {\em {SAT}}, volume 4996 of {\em LNCS}, pages 63--76, 2008.

\bibitem{Feller71-vol1}
William Feller.
\newblock {\em An Introduction to probability theory and its applications}, volume~1.
\newblock John Wiley {\&} Sons, Inc., 3 edition, 1968.

\bibitem{Feller71}
William Feller.
\newblock {\em An Introduction to probability theory and its applications}, volume~2.
\newblock John Wiley {\&} Sons, Inc., 2 edition, 1971.

\bibitem{Fichte2015}
Johannes~Klaus Fichte and Stefan Szeider.
\newblock Backdoors to tractable answer set programming.
\newblock {\em Artif. Intell.}, 220:64--103, 2015.

\bibitem{Gaspers2012a}
Serge Gaspers and Stefan Szeider.
\newblock Backdoors to acyclic {SAT}.
\newblock In {\em {ICALP} {(1)}}, volume 7391 of {\em LNCS}, pages 363--374, 2012.

\bibitem{FPTBook12}
Serge Gaspers and Stefan Szeider.
\newblock Backdoors to satisfaction.
\newblock In {\em The Multivariate Algorithmic Revolution and Beyond}, volume 7370 of {\em LNCS}, pages 287--317, 2012.

\bibitem{Gaspers2012c}
Serge Gaspers and Stefan Szeider.
\newblock Strong backdoors to nested satisfiability.
\newblock In {\em {SAT}}, volume 7317 of {\em LNCS}, pages 72--85, 2012.

\bibitem{Gates79}
William~H. Gates and Christos~H. Papadimitriou.
\newblock Bounds for sorting by prefix reversal.
\newblock {\em Discret. Math.}, 27(1):47--57, 1979.

\bibitem{Goldberg03}
Evguenii~I. Goldberg and Yakov Novikov.
\newblock Verification of proofs of unsatisfiability for {CNF} formulas.
\newblock In {\em {DATE}}, pages 10886--10891. {IEEE} Computer Society, 2003.

\bibitem{GomesSabh2021}
Carla~P. Gomes and Ashish Sabharwal.
\newblock Exploiting runtime variation in complete solvers.
\newblock In {\em Handbook of Satisfiability (Second Edition)}, volume 185 of {\em Frontiers in Artificial Intelligence and Applications}, pages 463--480. {IOS} Press, 2 edition, 2021.

\bibitem{Gribanova2018}
Irina Gribanova and Alexander Semenov.
\newblock Using automatic generation of relaxation constraints to improve the preimage attack on 39-step {MD4}.
\newblock In {\em MIPRO}, pages 1174--1179, 2018.

\bibitem{Gribanova20}
Irina Gribanova and Alexander Semenov.
\newblock Constructing a set of weak values for full-round {MD4} hash function.
\newblock In {\em MIPRO}, pages 1212--1217, 2020.

\bibitem{Gu1996}
Jun Gu, Paul~W. Purdom, John~V. Franco, and Benjamin~W. Wah.
\newblock {\em Algorithms for the satisfiability {(SAT)} problem: {A} survey}, volume~35 of {\em {DIMACS} Series in Discrete Mathematics and Theoretical Computer Science}, pages 19--151.
\newblock {DIMACS/AMS}, 1996.

\bibitem{Haken1985}
Armin Haken.
\newblock The intractability of resolution.
\newblock {\em Theor. Comput. Sci.}, 39:297--308, 1985.

\bibitem{Hemaspaandra2021}
Lane~A. Hemaspaandra and David~E. Narv\'{a}ez.
\newblock Existence versus exploitation: The opacity of backdoors and backbones.
\newblock {\em Prog. in Artif. Intell.}, 10(3):297–308, 2021.

\bibitem{drat-trim-2}
Marijn Heule, Warren A.~Hunt Jr., and Nathan Wetzler.
\newblock Trimming while checking clausal proofs.
\newblock In {\em {FMCAD}}, pages 181--188. {IEEE}, 2013.

\bibitem{CC2012}
Marijn Heule, Oliver Kullmann, Siert Wieringa, and Armin Biere.
\newblock Cube and conquer: Guiding {CDCL} {SAT} solvers by lookaheads.
\newblock In {\em HVC}, volume 7261 of {\em LNCS}, pages 50--65, 2011.

\bibitem{Heule18}
Marijn J.~H. Heule.
\newblock Schur number five.
\newblock In {\em {AAAI}}, pages 6598--6606. {AAAI} Press, 2018.

\bibitem{Heule16}
Marijn J.~H. Heule, Oliver Kullmann, and Victor~W. Marek.
\newblock Solving and verifying the boolean pythagorean triples problem via cube-and-conquer.
\newblock In {\em {SAT}}, volume 9710 of {\em LNCS}, pages 228--245, 2016.

\bibitem{Heule09}
Marijn J.~H. Heule and Hans van Maaren.
\newblock Look-ahead based {SAT} solvers.
\newblock In {\em Handbook of Satisfiability}, volume 336 of {\em Frontiers in Artificial Intelligence and Applications}, pages 183--212. {IOS} Press, 2021.

\bibitem{Hoeff63}
Wassily Hoeffding.
\newblock Probability inequalities for sums of bounded random variables.
\newblock {\em J. Am. Stat. Assoc.}, 58:13--30, 1963.

\bibitem{HyvarPhD}
A.~E.~J. Hyv{\"a}rinen.
\newblock {Grid Based Propositional Satisfiability Solving}, 2011.
\newblock {PhD thesis}.

\bibitem{Hyvar06}
Antti Eero~Johannes Hyv{\"{a}}rinen, Tommi~A. Junttila, and Ilkka Niemel{\"{a}}.
\newblock A distribution method for solving {SAT} in grids.
\newblock In {\em {SAT}}, volume 4121 of {\em LNCS}, pages 430--435, 2006.

\bibitem{Hyvar10}
Antti Eero~Johannes Hyv{\"{a}}rinen, Tommi~A. Junttila, and Ilkka Niemel{\"{a}}.
\newblock Partitioning {SAT} instances for distributed solving.
\newblock In {\em {LPAR}}, volume 6397 of {\em LNCS}, pages 372--386, 2010.

\bibitem{Hyvar11}
Antti Eero~Johannes Hyv{\"{a}}rinen, Tommi~A. Junttila, and Ilkka Niemel{\"{a}}.
\newblock Grid-based {SAT} solving with iterative partitioning and clause learning.
\newblock In {\em {CP}}, volume 6876 of {\em LNCS}, pages 385--399, 2011.

\bibitem{jia20}
Kai Jia and Martin~C. Rinard.
\newblock Efficient exact verification of binarized neural networks.
\newblock In {\em NeurIPS}, 2020.

\bibitem{KarpLuby09}
Richard~M. Karp, Michael Luby, and Neal Madras.
\newblock Monte-carlo approximation algorithms for enumeration problems.
\newblock {\em J. Algorithms}, 10(3):429--448, 1989.

\bibitem{KonLisitsa14}
Boris Konev and Alexei Lisitsa.
\newblock A {SAT} attack on the erd{\H{o}}s discrepancy conjecture.
\newblock In {\em {SAT}}, volume 8561 of {\em LNCS}, pages 219--226, 2014.

\bibitem{Kroening09}
Daniel Kroening.
\newblock Software verification.
\newblock In {\em Handbook of Satisfiability (Second Edition)}, pages 791--818. IOS Press, 2021.

\bibitem{Luke15}
Sean Luke.
\newblock {\em Essentials of Metaheuristics}.
\newblock Lulu, second edition, 2013.

\bibitem{LynceMS06}
In{\^{e}}s Lynce and Jo{\~{a}}o Marques{-}Silva.
\newblock {SAT} in bioinformatics: Making the case with haplotype inference.
\newblock In {\em {SAT}}, volume 4121 of {\em LNCS}, pages 136--141, 2006.

\bibitem{MannWhitney}
H.~B. Mann and D.~R. Whitney.
\newblock {On a Test of Whether one of Two Random Variables is Stochastically Larger than the Other}.
\newblock {\em Ann. Math. Stat.}, 18(1):50 -- 60, 1947.

\bibitem{MS_Handb09}
Joao Marques{-}Silva, Ines Lynce, and Sharad Malik.
\newblock Conflict-driven clause learning {SAT} solvers.
\newblock In {\em Handbook of Satisfiability (Second Edition)}, volume 185 of {\em Frontiers in Artificial Intelligence and Applications}, pages 133--182. {IOS} Press, 2021.

\bibitem{MSS1999}
J.P. Marques-Silva and K.A. Sakallah.
\newblock {GRASP}: a search algorithm for propositional satisfiability.
\newblock {\em IEEE Trans. Comput.}, 48(5):506--521, 1999.

\bibitem{Mass2000}
Fabio Massacci and Laura Marraro.
\newblock Logical cryptanalysis as a {SAT} problem.
\newblock {\em J. Autom. Reason.}, 24(1/2):165--203, 2000.

\bibitem{MU49}
Nicholas Metropolis and S.~Ulam.
\newblock {The Monte Carlo Method}.
\newblock {\em J. Amer. Statistical Assoc.}, 44(247):335--341, 1949.

\bibitem{Mironov2006}
Ilya Mironov and Lintao Zhang.
\newblock Applications of {SAT} solvers to cryptanalysis of hash functions.
\newblock In {\em {SAT}}, volume 4121 of {\em LNCS}, pages 102--115, 2006.

\bibitem{Misra2013}
Neeldhara Misra, Sebastian Ordyniak, Venkatesh Raman, and Stefan Szeider.
\newblock Upper and lower bounds for weak backdoor set detection.
\newblock In {\em {SAT}}, volume 7962 of {\em LNCS}, pages 394--402. Springer, 2013.

\bibitem{MotwRagh95}
Rajeev Motwani and Prabhakar Raghavan.
\newblock {\em Randomized Algorithms}.
\newblock Cambridge University Press, 1995.

\bibitem{narod18-ijcai}
Nina Narodytska.
\newblock Formal analysis of deep binarized neural networks.
\newblock In {\em IJCAI}, page 5692–5696, 2018.

\bibitem{narod18-aaai}
Nina Narodytska, Shiva~Prasad Kasiviswanathan, Leonid Ryzhyk, Mooly Sagiv, and Toby Walsh.
\newblock Verifying properties of binarized deep neural networks.
\newblock In {\em {AAAI}}, pages 6615--6624. {AAAI} Press, 2018.

\bibitem{evostar}
Artem Pavlenko, Alexander~A. Semenov, and Vladimir Ulyantsev.
\newblock Evolutionary computation techniques for constructing sat-based attacks in algebraic cryptanalysis.
\newblock In {\em EvoApplications}, volume 11454 of {\em LNCS}, pages 237--253. Springer, 2019.

\bibitem{SAT-at-home-2012}
Mikhail Posypkin, Alexander~A. Semenov, and Oleg Zaikin.
\newblock Using {BOINC} desktop grid to solve large scale {SAT} problems.
\newblock {\em Comput. Sci.}, 13(1):25--34, 2012.

\bibitem{Robins1965}
J.~A. Robinson.
\newblock A machine-oriented logic based on the resolution principle.
\newblock {\em J. ACM}, 12(1):23–41, 1965.

\bibitem{Schoning1999}
Uwe Sch{\"{o}}ning.
\newblock A probabilistic algorithm for k-sat and constraint satisfaction problems.
\newblock In {\em {FOCS}}, pages 410--414. {IEEE} Computer Society, 1999.

\bibitem{Selman1992}
Bart Selman, Hector~J. Levesque, and David~G. Mitchell.
\newblock A new method for solving hard satisfiability problems.
\newblock In {\em {AAAI}}, pages 440--446. {AAAI} Press / The {MIT} Press, 1992.

\bibitem{CP2021}
Alexander Semenov, Daniil Chivilikhin, Artem Pavlenko, Ilya Otpuschennikov, Vladimir Ulyantsev, and Alexey Ignatiev.
\newblock Evaluating the hardness of {SAT} instances using evolutionary optimization algorithms.
\newblock In {\em CP}, volume 210 of {\em LIPIcs}, pages 47:1--47:18, 2021.

\bibitem{SZ16}
Alexander Semenov and Oleg Zaikin.
\newblock Algorithm for finding partitionings of hard variants of boolean satisfiability problem with application to inversion of some cryptographic functions.
\newblock {\em SpringerPlus}, 5(1), 2016.
\newblock Article no. 554.

\bibitem{AAAI2022}
Alexander~A. Semenov, Artem Pavlenko, Daniil Chivilikhin, and Stepan Kochemazov.
\newblock On probabilistic generalization of backdoors in boolean satisfiability.
\newblock In {\em {AAAI}}, pages 10353--10361. {AAAI} Press, 2022.

\bibitem{SZ15}
Alexander~A. Semenov and Oleg Zaikin.
\newblock Using monte carlo method for searching partitionings of hard variants of boolean satisfiability problem.
\newblock In {\em PaCT}, volume 9251 of {\em LNCS}, pages 222--230, 2015.

\bibitem{SZBP11}
Alexander~A. Semenov, Oleg Zaikin, Dmitry Bespalov, and Mikhail Posypkin.
\newblock Parallel logical cryptanalysis of the generator {A5/1} in bnb-grid system.
\newblock In {\em PaCT}, volume 6873 of {\em LNCS}, pages 473--483, 2011.

\bibitem{SZOKI18}
Alexander~A. Semenov, Oleg Zaikin, Ilya~V. Otpuschennikov, Stepan Kochemazov, and Alexey Ignatiev.
\newblock On cryptographic attacks using backdoors for {SAT}.
\newblock In {\em {AAAI}}, pages 6641--6648, 2018.

\bibitem{MSS96}
Jo{\~{a}}o P.~Marques Silva and Karem~A. Sakallah.
\newblock {GRASP} - a new search algorithm for satisfiability.
\newblock In {\em {ICCAD}}, pages 220--227. {IEEE} Computer Society / {ACM}, 1996.

\bibitem{Soos09}
Mate Soos, Karsten Nohl, and Claude Castelluccia.
\newblock Extending {SAT} solvers to cryptographic problems.
\newblock In {\em {SAT}}, volume 5584 of {\em LNCS}, pages 244--257, 2009.

\bibitem{Spence15}
Ivor T.~A. Spence.
\newblock Weakening cardinality constraints creates harder satisfiability benchmarks.
\newblock {\em {ACM} J. Exp. Algorithmics}, 20:1.4:1--1.4:14, 2015.

\bibitem{Tse70}
Grigoriy Tseitin.
\newblock On the complexity of derivation in propositional calculus.
\newblock {\em Studies in Constr. Math. and Math. Logic}, pages 115--125, 1970.

\bibitem{urquhart1995}
Alasdair Urquhart.
\newblock The complexity of propositional proofs.
\newblock {\em Bull. Symb. Log.}, 1(4):425--467, 1995.

\bibitem{drat-trim-1}
Nathan Wetzler, Marijn Heule, and Warren A.~Hunt Jr.
\newblock Drat-trim: Efficient checking and trimming using expressive clausal proofs.
\newblock In {\em {SAT}}, volume 8561 of {\em LNCS}, pages 422--429, 2014.

\bibitem{Wilks62}
Samuel~S. Wilks.
\newblock {\em Mathematical statistics}.
\newblock John Wiley and Sons, 1962.

\bibitem{Williams03}
Ryan Williams, Carla~P. Gomes, and Bart Selman.
\newblock Backdoors to typical case complexity.
\newblock In {\em {IJCAI}}, pages 1173--1178. Morgan Kaufmann, 2003.

\bibitem{Zaikin2022}
O.~Zaikin.
\newblock Inverting 43-step of {MD4} using cube-and-conquer.
\newblock In {\em IJCAI}, pages 1894--1900, 2022.

\end{thebibliography}


\end{document}